\documentclass[11pt, a4paper]{chili}

\usepackage{parskip}
\usepackage{natbib}
\usepackage{enumitem}
\usepackage[table]{xcolor}
\usepackage{amsmath, amssymb}
\usepackage{booktabs}
\usepackage{multirow}
\usepackage{wrapfig}
\usepackage{lmodern}
\usepackage{geometry}
\usepackage{fontawesome5}  %
\usepackage{array}
\usepackage{makecell}

\usepackage{graphicx}
\usepackage{subcaption}
\usepackage{pifont}
\usepackage{float}
\usepackage{pgfplots}
\pgfplotsset{compat=1.18}
\usepackage{tikz}
\usetikzlibrary{patterns, positioning}
\usepackage{hyperref}
\usepackage[most]{tcolorbox}
\usepackage[nameinlink,capitalise]{cleveref}
\usepackage{amsthm}

\newtheorem{theorem}{Theorem}[section]
\newtheorem{proposition}[theorem]{Proposition}

\newtheorem{lemma}{Lemma}
\theoremstyle{remark}

\newlist{assetlist}{itemize}{1}
\setlist[assetlist]{
    leftmargin=1.8em,
    labelsep=0.6em,
    itemsep=0.3em,
    topsep=0.3em
}

\definecolor{lt2russet}{HTML}{8E3B30}
\definecolor{lt2grey}{RGB}{245,245,245}
\definecolor{MorandiRose}{RGB}{180,120,120}
\definecolor{lt2cream}{RGB}{251,235,211}     %
\definecolor{lt2tan}{RGB}{197,152,105}
\definecolor{lt2bandblue}{RGB}{232,240,250}  %
\definecolor{lt2good}{RGB}{34,120,60}
\definecolor{lt2bad}{RGB}{180,40,40}
\definecolor{lt2navy}{RGB}{27,58,111}        %
\definecolor{lt2blue}{RGB}{70,130,180}       %
\definecolor{lt2cream}{RGB}{251,235,211}     %
\definecolor{lt2tan}{RGB}{197,152,105}       %
\definecolor{lt2bandblue}{RGB}{232,240,250}  %
\definecolor{lt2bandred}{RGB}{252,238,232}   %
\definecolor{lt2bandgreen}{RGB}{231,243,235} %
\definecolor{lt2good}{RGB}{34,120,60}        %
\definecolor{lt2bad}{RGB}{180,40,40}         %
\definecolor{lt2cell}{RGB}{210,228,210}      %
\definecolor{lt2cellbad}{RGB}{248,220,220}   %
\definecolor{ourshl}{RGB}{235, 244, 255}    %
\definecolor{teacherhl}{RGB}{245, 245, 245} %

\newcommand{\hero}{\rowcolor{lt2cream}}

\newcommand{\bad}[1]{\textcolor{lt2bad}{#1}}
\newcommand{\yes}{\textcolor{lt2good}{\ding{51}}}
\newcommand{\no}{\textcolor{lt2bad}{\ding{55}}}

\newcommand{\gateop}[1]{\textcolor{lt2navy}{#1}}
\newcommand{\deltaruleop}[1]{\textcolor{MorandiRose}{#1}}

\newcommand{\hflogo}{
    \raisebox{-0.2em}{
        \includegraphics[height=1.2em]{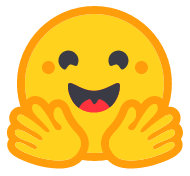}
    }
}
\newcommand{\ghlogo}{
    \raisebox{-0.2em}{
        \includegraphics[height=1.2em]{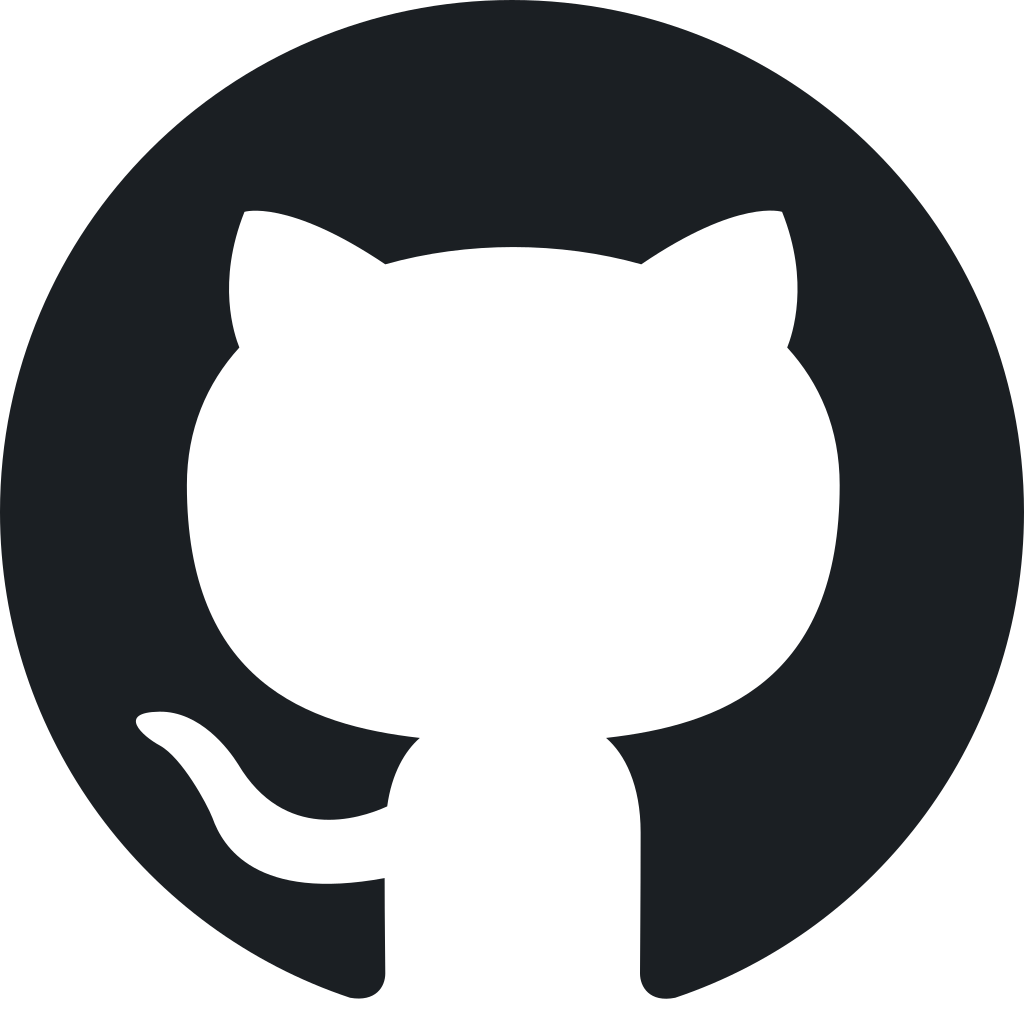}
    }
}
\begin{document}
\AtBeginDocument{
\hypersetup{
    urlcolor=MorandiRose
}
}
\makeatletter
\patchcmd{\archer@maketitle@onecolumn}
  {\vskip11pt}
  {\vskip3pt}
  {}{}
\makeatother

\title{LT2: Linear-Time Looped Transformers}

\author[1]{Chunyuan Deng}
\author[2]{Yizhe Zhang}
\author[3]{Rui-jie Zhu}
\author[1]{Yuanyuan Xu}
\author[4]{Jiarui Liu} 
\author[1]{T. S. Eugene Ng}
\author[1]{Hanjie Chen}

\affil[1]{Rice University}
\affil[2]{Apple}
\affil[3]{UC Santa Cruz}
\affil[4]{Carnegie Mellon University}
\correspondingauthor={chunyuan.deng@rice.edu, yizzhang@apple.com, 
ridger@live.cn, yx102@rice.edu, 
jiaruil5@andrew.cmu.edu, hanjie@rice.edu}

\maketitle

\vspace{-0.9em}
\begin{abstract}
Looped Transformers (LT) have emerged as a powerful architecture by iterating their layers multiple times before decoding the final token. However, their pairing with full attention retains quadratic complexity making it computationally expensive and slow. We introduce \textbf{LT2 (Linear-Time Looped Transformers)}, a family of looped architectures that replace quadratic softmax attention with subquadratic attention with linear-time complexity. We study two variants: \textit{LT2-linear} with linear attention and \textit{LT2-sparse} with sparse attention. We find looping uniquely synergizes with these variants: it enables iterative memory refinement in linear attention and progressively expands the effective receptive field in sparse attention. We formalize these benefits theoretically and demonstrate consistent empirical gains across controlled recall, state-tracking, and language modeling tasks. We then explore \textbf{LT2-hybrid}, a hybrid architecture that combines different attention variants in a looped setting. We find two architectural variants promising: (1) LT2-hybrid (GDN+DSA), which interleaves linear and sparse attention to \textit{maximize efficiency}, matching the standard looped transformer's quality at fully linear-time cost. and (2) LT2-hybrid (Full+GDN), which interleaves GDN with a small fraction of full attention layers to \textit{maximize quality}, surpassing the standard looped transformer in both performance and efficiency. Furthermore, we also show how to turn a pre-trained LT into an LT2-hybrid model. With only about 1B tokens of training, our converted model (\textit{Ouro-hybrid-1.4B}) outperforms industry-level 1B models and is competitive with industry-level 4B models while keeping the speed benefits of linear-time attention. Together, these two directions show a clear path to making looped transformers a more scalable architecture for language modeling and advancing the development of efficient, capable small language models.

\begin{itemize}[
    leftmargin=-0.4em,
    labelsep=-0.2em,
    itemsep=0.6em,
    align=left
]
    \item[\ghlogo] \textbf{Codebase}:
    \href{https://github.com/chili-lab/LT2}{\textbf{https://github.com/chili-lab/LT2}}

    \item[\hflogo] \textbf{\textcolor{lt2tan!90!black}{Huggingface Checkpoints}}: \href{https://huggingface.co/chili-lab/Ouro-hybrid-1.4B}{\textbf{https://huggingface.co/chili-lab/Ouro-hybrid-1.4B}}
    
\end{itemize}
\end{abstract}

\begin{figure}[H]
\vspace{-0.4cm}
    \centering
    \includegraphics[width=\linewidth]{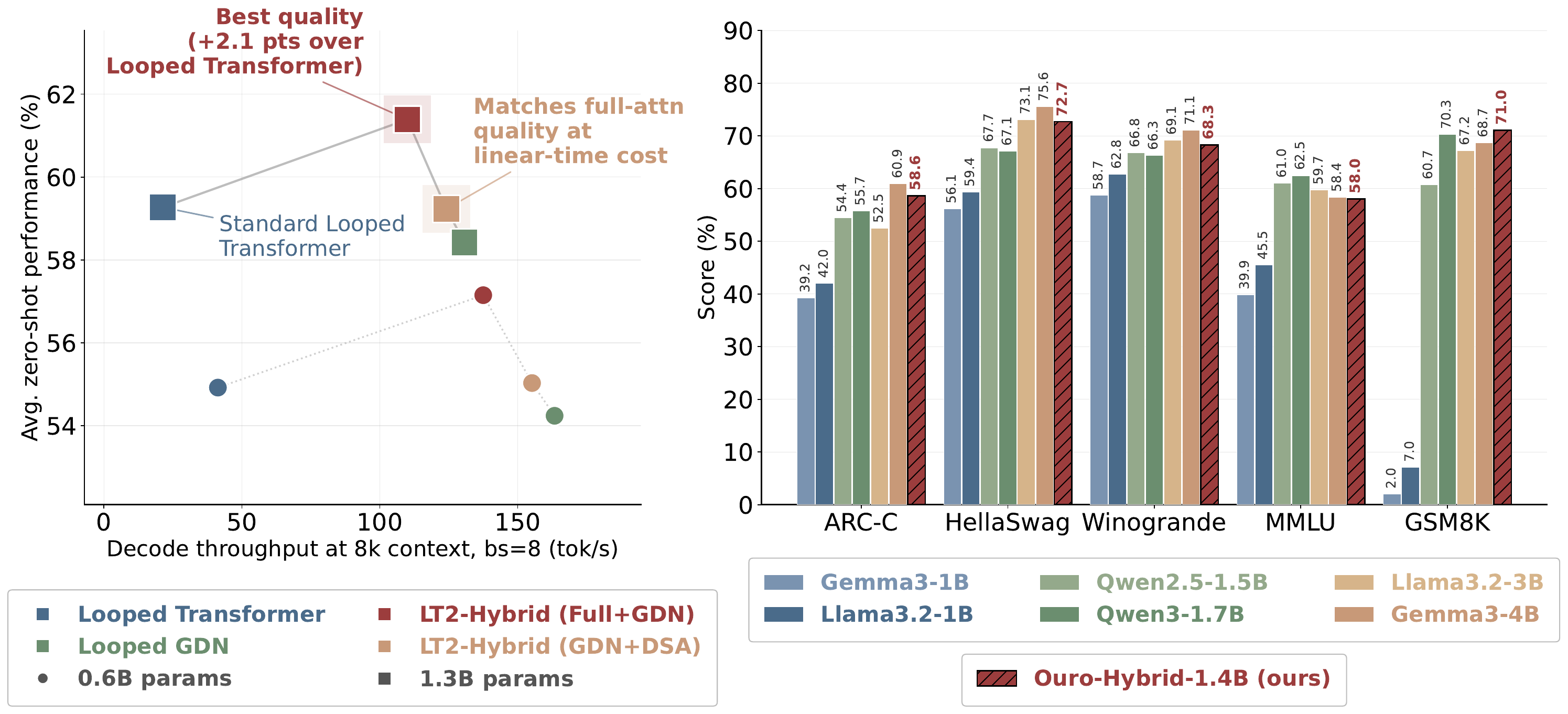}
    \caption{(\textbf{Left}) New parameter-efficiency frontier introduced by LT2. (\textbf{Right}) Converted LT2-Hybrid outperforms similarly sized industry-level 1B while matching 4B ones.}
\label{fig:main}
    \vspace{-0.8cm}
\end{figure}
\section{Introduction}

Scaling neural language models along the parameter axis has driven much of modern NLP's progress~\citep{brown2020languagemodelsfewshotlearners,kaplan2020scalinglawsneurallanguage,hoffmann2022trainingcomputeoptimallargelanguage}. A complementary axis---scaling depth via weight-shared recurrence---has recently emerged as a promising alternative. These architectures, often called looped transformers (LT, originally \textit{Universal Transformers}~\cite{dehghani2019universaltransformers}), reuse the same weights across multiple steps before decoding the final prediction token~\cite{giannou2023looped,yanglooped,zhu2025scalinglatentreasoninglooped}. In effect, repeated computation becomes effective depth: the model performs several rounds of latent computation while keeping the unique parameter count fixed, making looped transformers an appealing approach to parameter-efficient reasoning.

\begin{wrapfigure}{r}{0.5\textwidth}
  \vspace{-1.0em}
  \centering
  \includegraphics[width=0.5\textwidth]{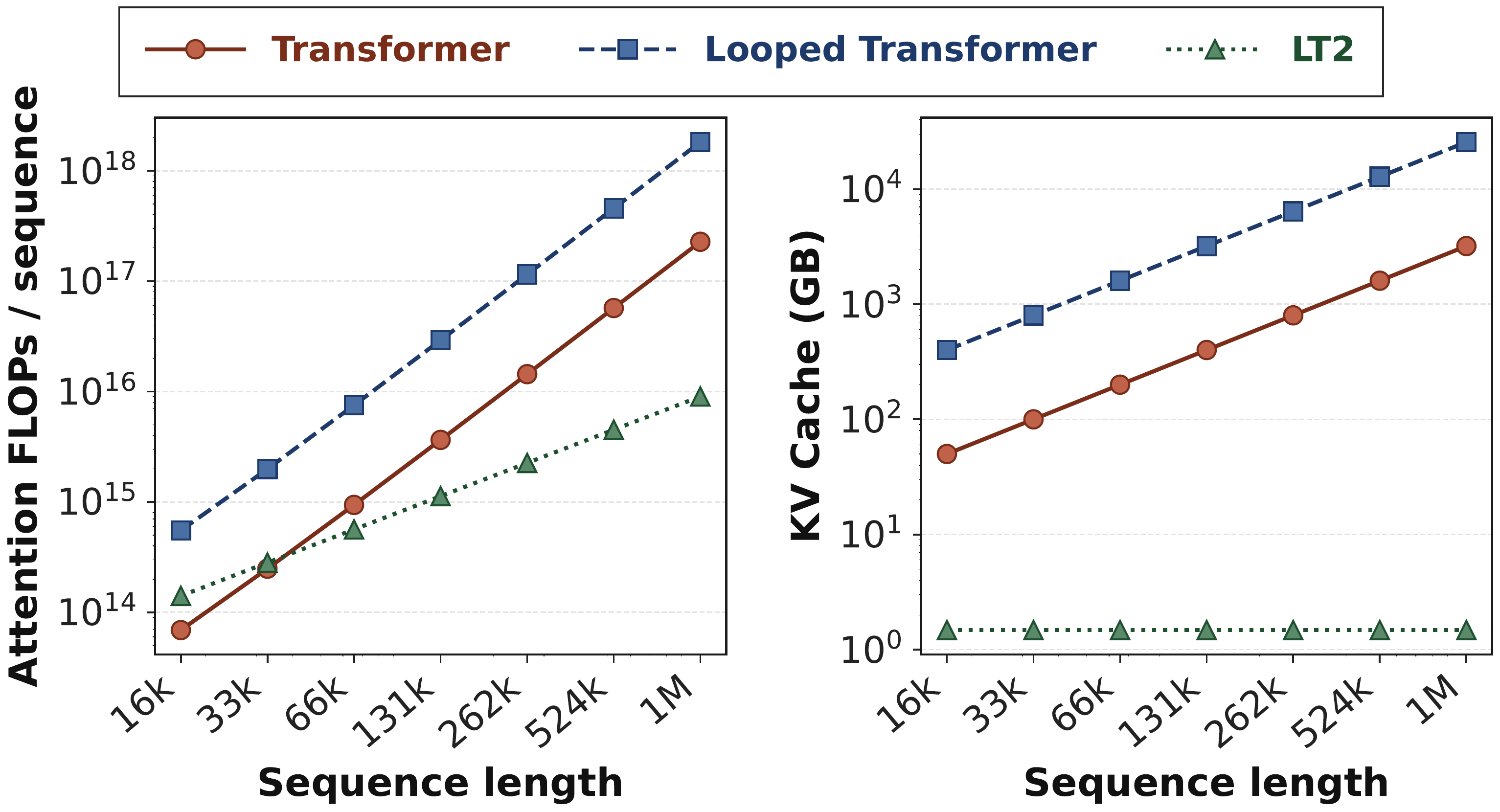}
  \vspace{-1.6em}
  \caption{Attention FLOPs and inference cache memory vs.\ sequence length for a
  $1.3\text{B}$ model.}
  \label{fig:flops_comparison}
  \vspace{-1.0em}
\end{wrapfigure}
However, current looped transformers scale poorly because each loop has to re-apply quadratic full attention over the entire sequence repeatedly. Its cost and inference-time storage therefore grow with sequence length, and compound with each loop iteration. As a result, even though parameters are reused, training-time attention FLOPs and inference-time KV-cache usage scale poorly with the number of loops, making attention the dominant bottleneck in scaling looped transformers~\cite{tay-etal-2023-scaling,zhu2025scalinglatentreasoninglooped}. As Figure~\ref{fig:flops_comparison} shows, processing every token through attention for $T$ iterations causes both training-time attention FLOPs and inference-time KV-cache memory to grow substantially. At long contexts the quadratic attention term dominates, and adding loop steps quickly becomes impractical~\cite{zhu2025scalinglatentreasoninglooped}.

We introduce \textbf{LT2 (Linear-Time Looped Transformers)}, a family of looped architectures that replace quadratic softmax attention with subquadratic token-mixing primitives. We primarily study two distinct variants, LT2-linear and LT2-sparse, which replace the quadratic attention with linear attention~\cite{katharopoulos2020transformers,yang2024parallelizing,kimiteam2025kimilinearexpressiveefficient} and sparse attention~\cite{xiao2024efficientstreaminglanguagemodels,deepseekai2025deepseekv32pushingfrontieropen}, respectively. We show that looped operation can \textit{turn compute into context}: it enables finer-grained control over recurrent memory in linear attention and enlarges the receptive field in sparse attention; we provide intuition in \S~\ref{sec:loop-gain} and a detailed theoretical analysis in Appendix~\ref{app:loop-dplr}. Furthermore, we explore \textbf{LT2-hybrid}, a hybrid architecture that pushes the performance--efficiency frontier to a new level by mixing different attention variants in the looped setting. We demonstrate that LT2-hybrid (GDN~\citep{yang2024gated} + DSA~\citep{deepseekai2025deepseekv32pushingfrontieropen})---which combines linear and sparse attention within a looped setting---matches the standard looped transformer's quality (59.3\% avg.\ zero-shot) while delivering $\sim$5.7$\times$ higher decode throughput at 8k context (125 vs.\ 22 tokens/s, batch size 8), entirely without quadratic attention. LT2-hybrid (Full + GDN), which interleaves GDN with a small fraction of full-attention layers, goes further: it improves average zero-shot performance by \textit{+2.1 points} over the standard looped transformer (61.4\% vs.\ 59.3\%) while still achieving $\sim$5$\times$ higher decode throughput at the same setting, and consistently outperforms the standard looped transformer across language modeling, recall, state-tracking, and efficiency benchmarks (\S~\ref{sec:experiments}).

Finally, we explore distilling a pretrained looped transformer (specifically, Ouro~\citep{zhu2025scalinglatentreasoninglooped}) into an LT2 model. As shown in Figure~\ref{fig:main} (right), with only $\sim$1B tokens of continued training, our converted \textit{Ouro-Hybrid-1.4B} retains the quality of its full-attention teacher while inheriting LT2's linear-time efficiency. The resulting model is competitive with industry-level open-source models in the 1B--4B parameter range across standard zero-shot benchmarks, matching or exceeding 1B-class baselines and approaching 3B--4B models on several tasks. This demonstrates that practitioners need not retrain from scratch: existing looped transformers can be efficiently converted into linear-time variants, lowering the cost barrier to adopting the LT2 family models.
\section{LT2: Linear-Time Looped Transformer}
\label{sec:method}
\subsection{Architecture}
\label{sec:method-lt2}
\textbf{Looped Transformer (LT).}
Let $L$ denote sequence length and $d$ the hidden dimension; we write the
hidden-state sequence as $\mathbf{h}\in\mathbb{R}^{L\times d}$ and the state at
position $t$ as $\mathbf{h}_t\in\mathbb{R}^{d}$. A standard Transformer of depth
$N$ stacks $N$ independently-parameterized blocks
$\{\mathcal{F}_\ell\}_{\ell=1}^{N}$, each consisting of a token mixer and a
position-wise FFN with residual connections:
\begin{equation}
  \mathcal{F}_\ell(\mathbf{h}) = \mathbf{h}' + \mathrm{FFN}_\ell(\mathbf{h}'),
  \qquad
  \mathbf{h}' = \mathbf{h} + \mathrm{MHA}_\ell(\mathbf{h}),
  \label{eq:transformer}
\end{equation}
where $\mathrm{MHA}_\ell$ is multi-head self-attention (we omit pre-norm for
brevity). A \emph{Looped Transformer} (LT) reuses these $N$ shared blocks for
$T$ iterations:
\begin{equation}
  \mathbf{h}^{(0)} = \mathrm{Emb}(\mathbf{x}),
  \quad
  \mathbf{h}^{(\tau)}
  =
  \bigl(\mathcal{F}_N \circ \cdots \circ \mathcal{F}_1\bigr)\!\bigl(\mathbf{h}^{(\tau-1)}\bigr),
  \quad \tau=1,\dots,T,
  \quad
  \hat{\mathbf{y}} = \mathrm{Dec}\!\bigl(\mathbf{h}^{(T)}\bigr),
  \label{eq:lt}
\end{equation}
yielding effective depth $T\cdot N$ with only $N$ unique parameter sets---a $T\times$
parameter reduction over a Transformer of equivalent depth. Following
Ouro~\citep{zhu2025scalinglatentreasoninglooped}, we use a fixed $T$ throughout
pre-training and we discuss adaptive computation time in the Appendix~\ref{app:act}.
 
\textbf{LT2.}
LT2 simply replaces the MHA sub-layer in Eq.~\eqref{eq:transformer} with a
subquadratic token mixer, so each shared block becomes
\begin{equation}
  \mathcal{F}_\ell(\mathbf{h}) = \mathbf{h}' + \mathrm{FFN}_\ell(\mathbf{h}'),
  \qquad
  \mathbf{h}' = \mathbf{h} + \mathrm{LinearMixer}_\ell(\mathbf{h}),
\end{equation}
where $\mathrm{LinearMixer}_\ell$ is any linear- or sparse-attention primitive
in Table~\ref{tab:mixers}. Throughout,
$\mathbf{q}_t,\mathbf{k}_t\!\in\!\mathbb{R}^{d_k}$ and
$\mathbf{v}_t\!\in\!\mathbb{R}^{d_v}$ denote the query/key/value projections of
$\mathbf{h}_t$; $\mathbf{S}_t\!\in\!\mathbb{R}^{d_k\times d_v}$ is the recurrent
state of a linear-attention mixer. We additionally insert a zero-initialized, per-channel learned gate
$\boldsymbol{\rho}_{\tau}\!\in\!\mathbb{R}^{d}$ as a residual across
loop iterations,
$\mathbf{h}^{(\tau)} = \widetilde{\mathbf{h}}^{(\tau)} + \boldsymbol{\rho}_{\tau}\odot\mathbf{h}^{(\tau-1)}$,
where $\widetilde{\mathbf{h}}^{(\tau)}$ is the output of the looped block stack at iteration $\tau$ (i.e., $\widetilde{\mathbf{h}}^{(\tau)} = (\mathcal{F}_N \circ \cdots \circ \mathcal{F}_1)(\mathbf{h}^{(\tau-1)})$). Thus our setup includes two levels of residual connections: a traditional per-block identity residual connection and a learned per-loop residual.

\begin{table*}[t]
\centering\small
\caption{Token mixers supported by LT2. Blue highlights \gateop{gating/retention} and burnt sienna highlights
\deltaruleop{DPLR-style} operations. Train FLOPs are reported per layer for a sequence of length $L$; cache/state memory is per layer at inference. $w$ denotes the sparse-attention window/budget size with $w \ll L$.}
\label{tab:mixers}
\resizebox{\textwidth}{!}{%
\begin{tabular}{@{}l l l c c@{}}
\toprule
\textbf{Family} & \textbf{Mixer} & \textbf{State update rule} & \textbf{Train FLOPs} & \textbf{Cache / State mem.} \\
\midrule
Full attn.
& Softmax MHA
& $(\mathbf{K}_{t},\mathbf{V}_{t})=\bigl([\mathbf{K}_{t-1};\mathbf{k}_{t}],\,[\mathbf{V}_{t-1};\mathbf{v}_{t}]\bigr)$
& $\mathcal{O}(L^{2} d)$
& $\mathcal{O}(L d)$ \\
\midrule
\multirow{8}{*}{Linear attn.\ (\textbf{LT2-LA})}
& LA~\citep{katharopoulos2020transformers}
& $\mathbf{S}_{t}=\mathbf{S}_{t-1}+\mathbf{k}_{t}\mathbf{v}_{t}^{\!\top}$
& $\mathcal{O}(L\,d_k d_v)$
& $\mathcal{O}(d_k d_v)$ \\
& RetNet~\citep{sun2023retentive}
& $\mathbf{S}_{t}=\gateop{\gamma}\,\mathbf{S}_{t-1}+\mathbf{k}_{t}\mathbf{v}_{t}^{\!\top}$
& $\mathcal{O}(L\,d_k d_v)$
& $\mathcal{O}(d_k d_v)$ \\
& Mamba2~\citep{dao2024transformersssmsgeneralizedmodels}
& $\mathbf{S}_{t}=\gateop{\alpha_{t}}\,\mathbf{S}_{t-1}+\mathbf{k}_{t}\mathbf{v}_{t}^{\!\top}$
& $\mathcal{O}(L\,d_k d_v)$
& $\mathcal{O}(d_k d_v)$ \\
& GLA~\citep{yanggated}
& $\mathbf{S}_{t}=\gateop{\mathrm{Diag}(\boldsymbol{\alpha}_{t})}\mathbf{S}_{t-1}
+\mathbf{k}_{t}\mathbf{v}_{t}^{\!\top}$
& $\mathcal{O}(L\,d_k d_v)$
& $\mathcal{O}(d_k d_v)$ \\
& HGRN2~\citep{qin2024hgrn2}
& $\mathbf{S}_{t}=\gateop{\mathrm{Diag}(\boldsymbol{\alpha}_{t})}\mathbf{S}_{t-1}
+\bigl(\mathbf{1}-\gateop{\boldsymbol{\alpha}_{t}}\bigr)\mathbf{v}_{t}^{\!\top}$
& $\mathcal{O}(L\,d_k d_v)$
& $\mathcal{O}(d_k d_v)$ \\
& DeltaNet~\citep{schlag2021linear,yang2024parallelizing}
& $\mathbf{S}_{t}=\deltaruleop{\bigl(\mathbf{I}-\beta_{t}\mathbf{k}_{t}\mathbf{k}_{t}^{\!\top}\bigr)}\mathbf{S}_{t-1}
+\deltaruleop{\beta_{t}\mathbf{k}_{t}\mathbf{v}_{t}^{\!\top}}$
& $\mathcal{O}(L\,d_k d_v)$
& $\mathcal{O}(d_k d_v)$ \\
& GDN~\citep{yang2024gated}
& $\mathbf{S}_{t}=\gateop{\alpha_{t}}\,
\deltaruleop{\bigl(\mathbf{I}-\beta_{t}\mathbf{k}_{t}\mathbf{k}_{t}^{\!\top}\bigr)}\mathbf{S}_{t-1}
+\deltaruleop{\beta_{t}\mathbf{k}_{t}\mathbf{v}_{t}^{\!\top}}$
& $\mathcal{O}(L\,d_k d_v)$
& $\mathcal{O}(d_k d_v)$ \\
& KDA~\citep{kimiteam2025kimilinearexpressiveefficient}
& $\mathbf{S}_{t}=\deltaruleop{\bigl(\mathbf{I}-\beta_{t}\mathbf{k}_{t}\mathbf{k}_{t}^{\!\top}\bigr)}
\gateop{\mathrm{Diag}(\boldsymbol{\alpha}_{t})}\mathbf{S}_{t-1}
+\deltaruleop{\beta_{t}\mathbf{k}_{t}\mathbf{v}_{t}^{\!\top}}$
& $\mathcal{O}(L\,d_k d_v)$
& $\mathcal{O}(d_k d_v)$ \\
\midrule
\multirow{3}{*}{Sparse attn.\ (\textbf{LT2-SA})}
& Window
& $(\mathbf{K}_{t},\mathbf{V}_{t})=(\mathbf{K}_{[t-w:t]},\mathbf{V}_{[t-w:t]})$ \ \ (sliding cache)
& $\mathcal{O}(L\,w\,d)$
& $\mathcal{O}(w\,d)$ \\
& NSA~\citep{yuan2025nativesparseattentionhardwarealigned}
& KV cache + compressed blocks; $\mathcal{I}_{t}$: top-$w$ selected indices
& $\mathcal{O}(L\,w\,d)$
& $\mathcal{O}(L\,d)$ \\
& DSA~\citep{deepseekai2025deepseekv32pushingfrontieropen}
& KV cache; $\mathcal{I}_{t}$: top-$w$ via lightning indexer
& $\mathcal{O}(L\,w\,d)$
& $\mathcal{O}(L\,d)$ \\
\bottomrule
\end{tabular}%
}
\vspace{-0.4cm}
\end{table*}

\subsection{Beyond Efficiency: Benefits of Looping}
\label{sec:loop-gain}

Subquadratic attention provides clear efficiency gains. A more interesting question is what looping adds to these attention variants. We make two claims: with $T$ loop iterations, a diagonal-plus-low-rank (DPLR) linear-attention block turns its rank-$1$ state update into a rank-$T$ update, and a sliding-window block turns its window of size $w$ into an effective receptive field of size $Tw$.

\paragraph{Loop $\times$ DPLR linear attention: rank-$T$ update on recurrent memory.}
Frontier linear-attention architectures now use DPLR
mixers, e.g.\ GDN~\citep{yang2024gated},
KDA~\citep{kimiteam2025kimilinearexpressiveefficient},
and RWKV7~\citep{peng2025rwkv7gooseexpressivedynamic}. We take KDA as our
running example, which maintains a recurrent state
$\mathbf{S}_t\!\in\!\mathbb{R}^{d_k\times d_v}$ at sequence position $t$ via
\begin{equation}
  \mathbf{S}_t
  =
  \mathbf{A}_t\,\mathbf{S}_{t-1}
  +
  \beta_t\,\mathbf{k}_t\mathbf{v}_t^{\top},
  \qquad
  \mathbf{A}_t
  =
  \mathrm{Diag}(\boldsymbol{\alpha}_t)
  \bigl(\mathbf{I}-\beta_t\,\mathbf{k}_t\mathbf{k}_t^{\top}\bigr),
  \label{eq:dplr}
\end{equation}
where $\boldsymbol{\alpha}_t\!\in\![0,1]^{d_k}$ is a diagonal gate, so
$\mathbf{A}_t$ is identity ($\mathbf{I}$) plus a rank-$1$ ($\mathbf{k}_t\mathbf{k}_t^{\top}$) perturbation. Prior work shows that a
single such block can only model permutations of two elements per token, and
cannot solve the $S_n$ word problem ($n\!\geq\!3$) in finite
precision~\citep{grazzi2025unlockingstatetrackinglinearrnns}. When the
\emph{same} shared block is looped $T$ times, each loop iteration
$\tau\!\in\!\{1,\dots,T\}$ contributes a fresh DPLR factor
$\mathbf{A}^{(\tau)}_{t}$ acting on the recurrent state at position $t$, so the
cumulative per-token state-transition operator becomes
\begin{equation}
  \mathbf{A}^{\mathrm{eff}}_{t}
  =
  \prod_{\tau=1}^{T}\mathbf{A}^{(\tau)}_{t}
  =
  \prod_{\tau=1}^{T}
  \mathrm{Diag}\!\bigl(\boldsymbol{\alpha}^{(\tau)}_{t}\bigr)
  \!\left(
    \mathbf{I}
    -
    \beta^{(\tau)}_{t}\,
    \mathbf{k}^{(\tau)}_{t}\mathbf{k}^{(\tau)\top}_{t}
  \right).
  \label{eq:loop-product}
\end{equation}
\emph{DeltaProduct}~\citep{siems2025deltaproductimprovingstatetrackinglinear}
shows that the expressivity gains of looped DPLR depend on the relationships
among the loop-specific keys $\{\mathbf{k}^{(\tau)}_{t}\}_{\tau=1}^{T}$. In one
extreme, if all keys are identical, then each loop iteration erases the same
direction in recurrent memory, yielding no expressivity gain over the
non-looped case. In the other extreme, if the keys from different loop
iterations are orthogonal, then the loop erases historical information along
$T$ distinct directions. In this case, the original transition, which contains a
single rank-$1$ perturbation, is replaced by an effective transition with a
rank-$T$ memory-erasure subspace. We detailed discuss the proof in
Appendix~\ref{app:loop-dplr}.

\paragraph{Loop $\times$ sparse attention: receptive-field expansion.}
A single sliding-window block with window $w$ lets each query at position $t$
attend only to the last $w$ tokens,
\[
  \mathcal{I}_{t}^{(1)} = \{t-w+1,\dots,t\},
\]
so the per-loop receptive field is $\mathcal{O}(w)$ and anything beyond the $w$
tokens is invisible. Looping the block re-runs the same window over the sequence,
and information moves further with every loop iteration: at loop
iteration $\tau$, position $t$ attends to a window of loop-$(\tau{-}1)$ states,
and those states have already absorbed information from their own windows at
loop iteration $\tau{-}2$, and so on. Chaining this argument inductively
(Appendix~\ref{app:loop-swa}) gives the receptive field after $T$ loop
iterations:
\begin{equation}
  \mathcal{I}_{t}^{(T)}
  \supseteq
  \bigl\{\max(1,\,t-Tw+1),\,\dots,\,t\bigr\},
  \qquad
  \bigl|\mathcal{I}_{t}^{(T)}\bigr|
  =
  \mathcal{O}(Tw).
  \label{eq:loop-swa}
\end{equation}
In other words, $T$ loops of a window-$w$ block reach as far back as $T$ stacked
layers of window-$w$ attention~\citep{chen2025powerattentionexponentiallyscalingreceptive},
but with $T\times$ fewer parameters. Looping therefore \emph{turns compute into
context}: once $T$ is moderately large, a small fixed window already covers long
sequences, which makes sparse mixers a natural partner for looping in
long-context settings.

\subsection{Hybrid LT2: Mixing Mixers Across Depth and Loops}
\label{sec:hybrid}

\begin{figure}[t]
  \vspace{-1.0em}
  \centering
  \includegraphics[width=0.9\textwidth]{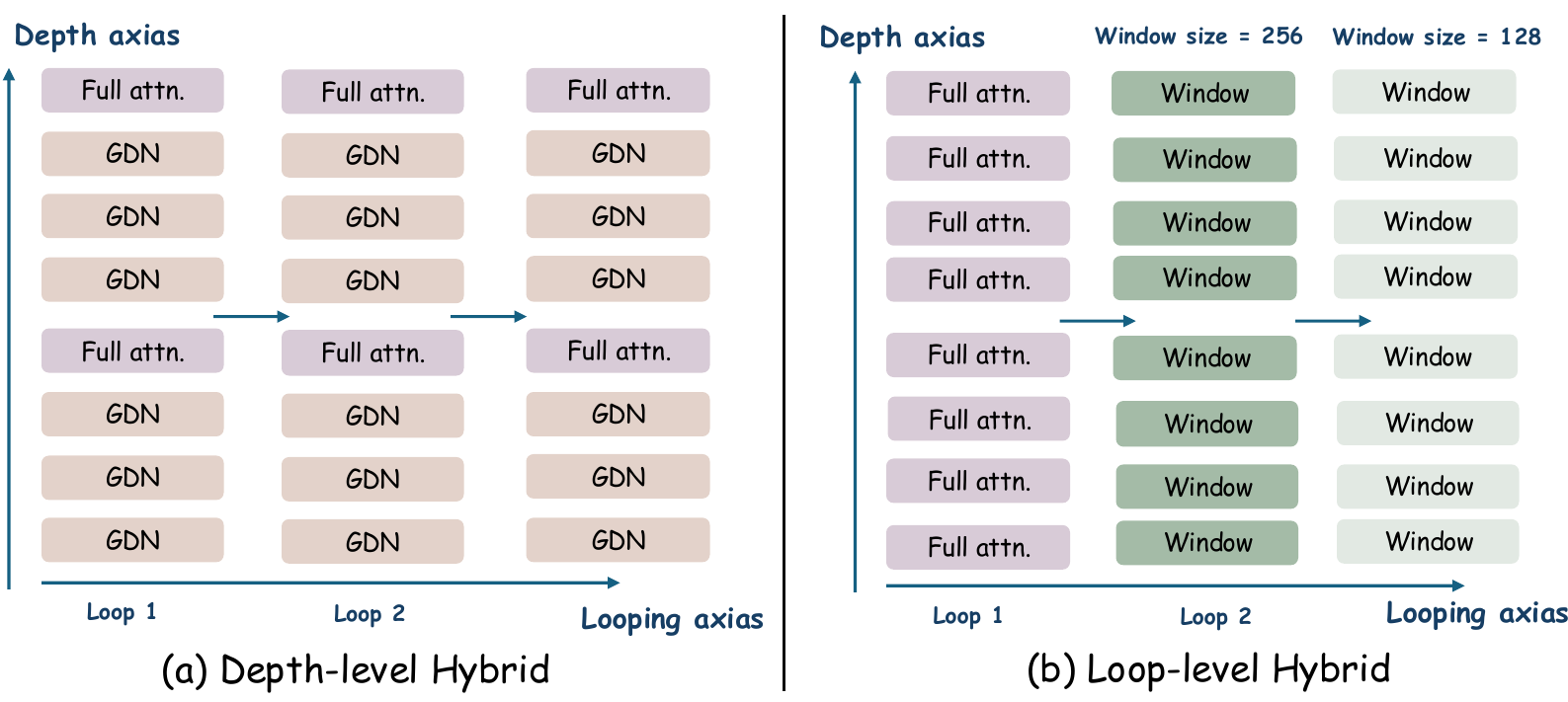}
  \vspace{-0.3em}
  \caption{Two ways to hybridize LT2. \textbf{(a) Depth-level} interleaves
  full-attention layers among linear layers inside the shared
  block. \textbf{(b) Loop-level}
  varies the mixer across loop
  iterations, e.g.\ a full-attention loop first, then sliding-window loops
  with shrinking windows ($256\!\to\!128$).}
  \label{fig:hybrid}
  \vspace{-0.4cm}
\end{figure}

We further explore hybrid architectures in the loop setting. A common practice
in hybrid models is to interleave linear blocks with full-attention blocks to
achieve strong language modeling performance while restoring recall
capability~\citep{lieber2024jambahybridtransformermambalanguage,
merrill2026olmohybridtheorypractice,qwenteam2026qwen35omnitechnicalreport,
nvidia2025nvidianemotron3efficient}. We show that looped transformers open a
\emph{second} axis for such mixing: beyond varying the mixer along \emph{depth},
we can also vary it across \emph{loop iterations}. We explore both options
(Figure~\ref{fig:hybrid}) in following section~\ref{sec:abl-hybrid-design}.

\section{Experiments}
\label{sec:experiments}

We organize the main experiments around four questions. First, we test whether LT2 is competitive at standard language-modeling scale (\S~\ref{sec:lm}) and under realistic long-context retrieval (\S~\ref{sec:long-context}). We then ablate the hybrid-design choices: where hybridization is applied, how mixers are arranged along depth, and what mixer ratio is used (\S~\ref{sec:abl-hybrid-design}). We then study the SDPA output gate, which mitigates attention-sink accumulation under looped processing (\S~\ref{sec:abl-output-gate}). Additional, we also did experiments cover synthetic recall/state tracking, long-context efficiency, and training stability (\S~\ref{sec:stability}, \ref{sec:efficiency}, and \ref{sec:recall-state}).

\subsection{Language modeling}
\label{sec:lm}

We pre-train all models on FineWeb-Edu~\citep{penedo2024finewebdatasetsdecantingweb} at 0.6B and 1.3B parameters with a 100B-token budget, using T=4 loops for every looped variant. The hybrid ratio is 1:4 for both (Full:Linear) and (GDN:DSA) variants. The full setup is in Appendix~\ref{app:lm-setup}.
Table~\ref{tab:downstream} summarizes the results. We provide a detailed efficiency comparison in Section~\ref{sec:efficiency}.

\paragraph{Subquadratic Mixers Nearly Match the Full-Attention Loop.}
Looped GDN, KDA, and DSA all perform within roughly one average point of the Looped Transformer reference at both scales, while avoiding quadratic complexity. At the smaller 0.6B scale, looped GDN still slightly trails the Looped Transformer; however, at the larger 1.3B scale, it surpasses the Looped Transformer while preserving linear-time complexity. In the looped setting, we find that both gating and DPLR linear attention are important, with gating appearing to play an even larger role than the DPLR-style update formulation. In contrast, the looped pure DeltaNet variant is less stable in our study, which ultimately limits its performance.

\paragraph{A Linear--Sparse Hybrid Loop Matches the Full-Attention Loop at a Fraction of the Cost.} Looped Hybrid (GDN+DSA), which contains no full attention, matches the full-attention reference at both scales ($9.72$ vs.\ $9.87$ PPL at $1.3$B). It also delivers the greatest efficiency speedup: a $2.9\times$ decode-throughput speedup at 32k context. We think this is an interesting hybrid setting where linear attention helps with global compression while sparse attention helps with exact KV position selection.

\begin{table}[H]
\vspace{-0.2cm}
\centering
\footnotesize
\caption{\textbf{Zero-shot downstream performance across two scales on FineWeb-Edu, $K{=}4$ loops.}
D-Gate = data-dependent gating; $\Delta$ = DPLR linear variants.
\colorbox{lt2cream}{Cream} marks the best LT2 model without full attention. Best per column \emph{within each scale} in \textbf{bold}, second-best \underline{underlined}.
Token budgets are relative to Chinchilla compute-optimal scaling~\cite{hoffmann2022trainingcomputeoptimallargelanguage}.}
\vspace{-0.2cm}
\label{tab:downstream}
\setlength{\tabcolsep}{3.2pt}
\renewcommand{\arraystretch}{0.95}
\begin{tabular}{l cc *{10}{r}}
\toprule
\textbf{Model} & D-Gate & $\Delta$ & PPL$(\downarrow)$ & ARC-E & ARC-C & HellaS. & PIQA & WG & OBQA & SciQ & BoolQ & Avg.$(\uparrow)$ \\
\midrule
\multicolumn{13}{l}{\textbf{0.6B parameters / 100B tokens (8$\times$ Chinchilla Ratio~\cite{hoffmann2022trainingcomputeoptimallargelanguage})}} \\
\midrule
Transformer                      & ---  & ---  & 13.14 & 63.09 & 30.72 & 47.43 & 69.53 & 56.24 & 35.6 & 68.2 & 50.07 & 51.34 \\
Looped Transformer (ref)         & ---  & ---  & 11.92 & 67.13 & 34.67 & 53.29 & 70.58 & 62.83 & \underline{38.2} & \textbf{73.6} & 54.87 & 56.42 \\
\addlinespace[2pt]
\rowcolor{lt2navy!8}\multicolumn{13}{l}{\textit{\textbf{\textcolor{lt2navy}{$\blacktriangleright$ LT2-linear attention}}}} \\
Looped RetNet~\citep{sun2023retentive}                    & \no  & \no  & ---   & \multicolumn{9}{c}{\bad{\itshape training diverged}} \\
Looped HGRN2~\citep{qin2024hgrn2}                         & \yes & \no  & 14.59 & 59.82 & 27.93 & 43.17 & 67.34 & 52.13 & 33.4 & 65.2 & 48.53 & 49.69 \\
Looped Mamba2~\citep{dao2024transformersssmsgeneralizedmodels} & \yes & \no  & 12.78 & 64.53 & 31.82 & 49.87 & 69.74 & 58.63 & 35.6 & 68.8 & 51.83 & 53.86 \\
Looped DeltaNet~\citep{schlag2021linear,yang2024parallelizing} & \no  & \yes & 14.16 & 60.47 & 28.53 & 44.22 & 67.87 & 53.24 & 33.8 & 65.5 & 49.13 & 50.12 \\
Looped GDN~\citep{yang2024gated}                 & \yes & \yes & 12.06 & 66.43 & 33.89 & 52.62 & 70.27 & 61.48 & 36.4 & 70.5 & 54.13 & 55.74 \\
Looped KDA~\citep{kimiteam2025kimilinearexpressiveefficient} & \yes & \yes & 12.13 & 66.12 & 33.63 & 52.37 & 70.13 & 61.22 & 36.2 & 70.2 & 53.92 & 55.49 \\
\addlinespace[2pt]
\rowcolor{lt2navy!8}\multicolumn{13}{l}{\textit{\textbf{\textcolor{lt2navy}{$\blacktriangleright$ LT2-sparse attention}}}} \\
Looped Window               & ---  & ---  & 12.87 & 64.23 & 31.53 & 48.83 & 69.82 & 57.34 & 35.8 & 68.5 & 51.23 & 52.17 \\
Looped NSA~\citep{yuan2025nativesparseattentionhardwarealigned} & ---  & ---  & 12.30 & 65.57 & 32.74 & 51.43 & 70.04 & 60.32 & 36.0 & 69.5 & 53.13 & 54.84 \\
Looped DSA~\citep{deepseekai2025deepseekv32pushingfrontieropen} & ---  & ---  & 12.08 & 66.37 & 33.82 & 52.53 & 70.23 & 61.42 & 36.4 & 70.4 & 54.07 & 55.67 \\
\addlinespace[2pt]
\rowcolor{lt2navy!8}\multicolumn{13}{l}{\textit{\textbf{\textcolor{lt2navy}{$\blacktriangleright$ Hybrid LT2 (linear/sparse/full permutation)}}}} \\
Looped Hybrid (Full+Window)           & ---  & ---  & 12.24 & 65.32 & 32.13 & 51.23 & 69.86 & 58.42 & 36.0 & 69.2 & 53.13 & 54.43 \\
Looped Hybrid (Full+DSA)              & ---  & ---  & 12.20 & 65.53 & 32.34 & 51.42 & 70.04 & 58.63 & 36.2 & 69.4 & 53.32 & 54.62 \\
Looped Hybrid (Full+GDN)
                                 & \yes & \yes & \textbf{11.43}
        & \textbf{69.82} & \textbf{37.34} & \textbf{55.83} & \textbf{72.62} & \textbf{64.61} & \textbf{38.9} & \underline{73.3} & \textbf{57.74} & \textbf{58.65} \\
\hero
Looped Hybrid (GDN+DSA)          & \yes & \yes & \underline{11.85} & \underline{67.43} & \underline{34.53} & \underline{53.42} & \underline{70.63} & \underline{62.92} & 37.0 & 71.2 & \underline{55.13} & \underline{56.53} \\

\midrule
\multicolumn{13}{l}{\textbf{1.3B parameters / 100B tokens (4$\times$ Chinchilla Ratio~\cite{hoffmann2022trainingcomputeoptimallargelanguage})}} \\
\midrule
Transformer                      & ---  & ---  & 10.65 & 67.52 & 33.84 & 52.47 & 71.03 & 61.48 & 36.6 & 71.3 & 54.02 & 56.04 \\
Looped Transformer (ref)         & ---  & ---  &  9.87 & 70.83 & 37.54 & 57.06 & 72.43 & 65.83 & 38.6 & 74.1 & 57.83 & 59.27 \\
\addlinespace[2pt]
\rowcolor{lt2navy!8}\multicolumn{13}{l}{\textit{\textbf{\textcolor{lt2navy}{$\blacktriangleright$ LT2-linear attention}}}} \\
Looped Mamba2~\citep{dao2024transformersssmsgeneralizedmodels} & \yes & \no  & 10.30 & 69.47 & 36.63 & 55.94 & 72.68 & 64.37 & 38.2 & 73.0 & 57.03 & 58.43 \\
Looped GDN~\citep{yang2024gated}                 & \yes & \yes &  9.75 & 71.28 & 38.33 & 57.73 & 73.37 & 66.26 & 39.1 & 74.3 & 58.78 & 59.92 \\
Looped KDA~\citep{kimiteam2025kimilinearexpressiveefficient} & \yes & \yes &  9.68 & 71.57 & 38.62 & 57.99 & 73.53 & 66.42 & 39.3 & 74.6 & 58.98 & 60.14 \\
\addlinespace[2pt]
\rowcolor{lt2navy!8}\multicolumn{13}{l}{\textit{\textbf{\textcolor{lt2navy}{$\blacktriangleright$ LT2-sparse attention}}}} \\
Looped Window              & ---  & ---  & 10.42 & 68.43 & 35.47 & 54.87 & 71.32 & 63.23 & 36.9 & 71.7 & 55.87 & 57.23 \\
Looped NSA~\citep{yuan2025nativesparseattentionhardwarealigned} & ---  & ---  & 10.17 & 69.02 & 35.97 & 55.08 & 71.52 & 64.03 & 37.2 & 72.2 & 56.53 & 57.72 \\
Looped DSA~\citep{deepseekai2025deepseekv32pushingfrontieropen} & ---  & ---  &  9.97 & 69.93 & 36.93 & 56.38 & 71.94 & 64.87 & 37.7 & 72.9 & 57.42 & 58.54 \\
\addlinespace[2pt]
\rowcolor{lt2navy!8}\multicolumn{13}{l}{\textit{\textbf{\textcolor{lt2navy}{$\blacktriangleright$ Hybrid LT2 (linear/sparse/full permutation)}}}} \\
Looped Hybrid (Full+Window)           & ---  & ---  &  9.84 & 70.93 & 37.12 & 56.68 & 73.12 & 64.34 & 38.8 & 73.3 & 58.56 & 59.13 \\
Looped Hybrid (Full+DSA)              & ---  & ---  &  9.80 & 71.13 & 37.28 & 56.84 & 73.24 & 64.52 & 38.9 & 73.4 & 58.73 & 59.28 \\
Looped Hybrid (Full+GDN)
                                 & \yes & \yes & \textbf{9.12}
        & \textbf{74.82} & \textbf{41.63} & \textbf{61.04} & \textbf{75.93} & \textbf{69.52} & \textbf{41.3} & \textbf{75.4} & \textbf{62.04} & \textbf{62.89} \\
\hero
Looped Hybrid (GDN+DSA)          & \yes & \yes &  \underline{9.50} & \underline{72.44} & \underline{39.33} & \underline{58.84} & \underline{73.98} & \underline{67.13} & \underline{39.7} & \underline{74.9} & \underline{59.77} & \underline{60.73} \\
\bottomrule
\end{tabular}
\vspace{-0.4cm}
\end{table}
\paragraph{Looped Hybrid (Full+GDN) Pushes to a New Pareto Frontier.} This is the strongest configuration overall, improving the general language modeling performance at both scales ($61.39$ vs.\ $59.27$ at $1.3$B), with the largest gains on harder reasoning task. Since only a small fraction of layers are quadratic, it still yields a $\times2.7$ decode speedup. Together, the two hybrids bracket the new Pareto frontier: one matches full attention at near-linear cost, the other exceeds it while staying markedly faster than the all-full-attention loop.

\subsection{Efficiency at long context}
\label{sec:efficiency}
We measure prefill and decode throughput for the four looped LT2
candidates from $1$k to $32$k tokens, at batch sizes $\{1,2,4,8\}$, on
a single H100 (80\,GB) with FlashAttention-2~\citep{dao2023flashattention2fasterattentionbetter} for softmax attention and
a fused chunkwise kernel for GDN. All variants use $T{=}4$ at matched
parameter count. Open squares mark the last length each configuration
fit in memory before going out-of-memory.

Two trends are visible across all four batch rows.
\begin{figure}[H]
    \centering
    \includegraphics[width=\textwidth]{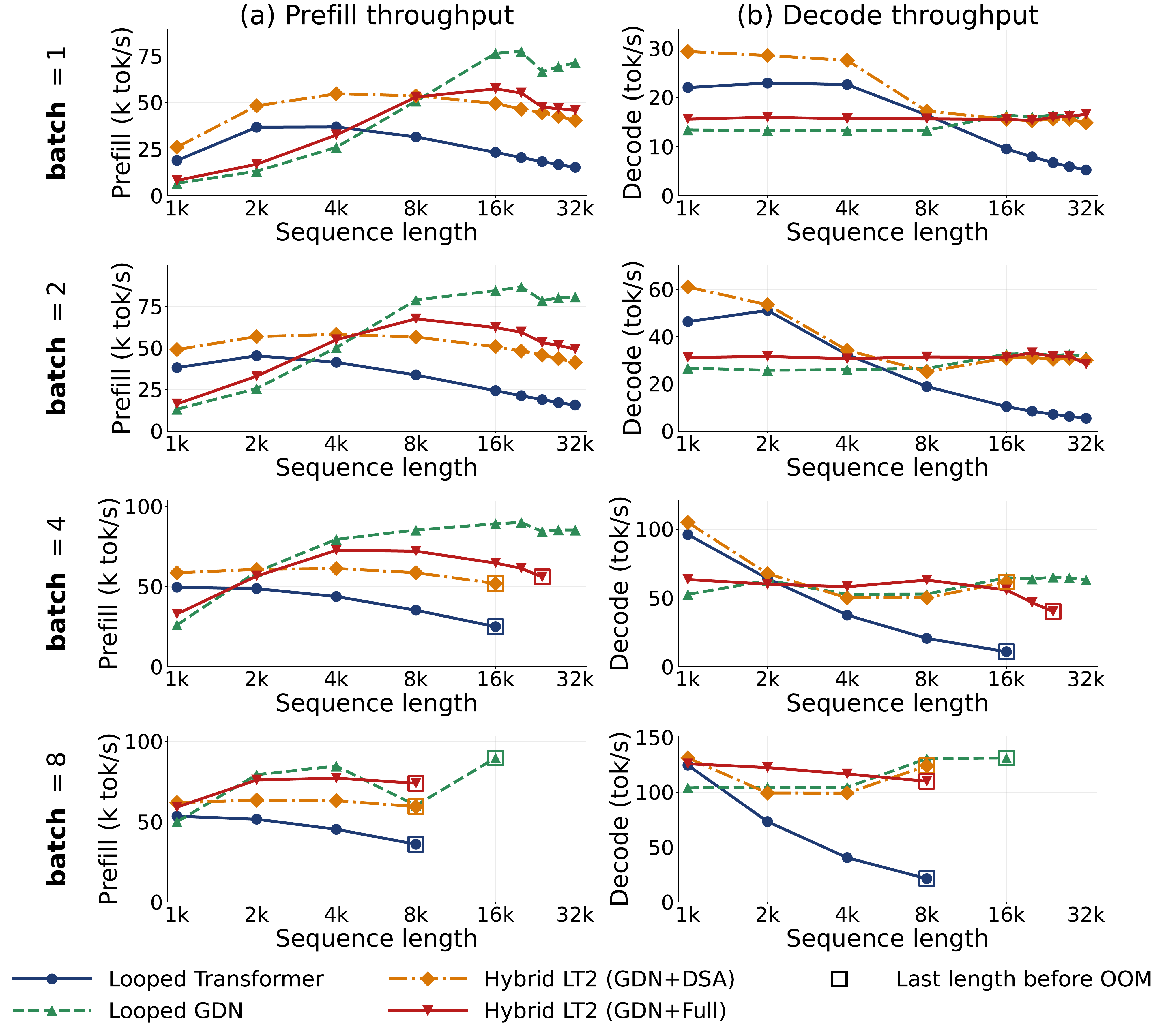}
    \vspace{-1.8em}
    \caption{\textbf{Efficiency at long context across batch sizes.}
    Rows are batch size ($1{,}2{,}4{,}8$); columns are
    \textbf{(a)} prefill and \textbf{(b)} decode throughput vs.\
    sequence length. Open squares mark the last sequence length each
    configuration reached before exhausting $80$\,GB of HBM. Looped
    GDN and Hybrid LT2 (GDN+Full) are the only variants that reach
    $32$k at every batch size and that hold their decode throughput
    flat across the full range; Hybrid LT2 (GDN+DSA) tracks them
    closely thanks to its top-$k$ KV reads.}
    \label{fig:efficiency}
    \vspace{-0.8em}
\end{figure}
\paragraph{Linear-time mixers eliminate the long-context decode cliff.}
Looped Transformer loses more than half of its decode throughput
between $4$k and $32$k because its KV cache keeps growing per loop
iteration. Looped GDN, Hybrid LT2 (GDN+Full), and Hybrid LT2 (GDN+DSA)
all hold a flat decode rate across the entire range and across all
batch sizes: at $\text{bs}{=}1$, $32$k they decode roughly $3\times$
faster than the LT, and at $\text{bs}{=}8$ they reach
$32$k while the LT OOMs by $8$k. The advantage
compounds with batch size, since each extra batch element adds a
fixed-size GDN state but a length-proportional KV cache.

\paragraph{Linear-time mixers also extend the OOM frontier.} As batch
grows, the Looped Transformer OOMs progressively earlier
($\text{bs}{=}4$: cap near $16$k, $\text{bs}{=}8$: cap near $8$k),
while Looped GDN reaches $32$k at every batch size and Hybrid LT2
(GDN+Full) reaches $32$k up to $\text{bs}{=}8$. Hybrid LT2 (GDN+DSA)
sits between the two, since DSA still maintains a KV cache for top-$k$
selection but only reads a small slice of it per query. In practice
this is the difference between serving long context at a useful batch
size and not.

\subsection{Ablation: hybrid ratio, pattern, and hybridization level}
\label{sec:abl-hybrid-design}

The hybrid LT2 in Section~\ref{sec:lm} fixes three design choices at once:
\emph{how much} attention sits in the loop, \emph{where} it sits along depth,
and \emph{at what level} mixers are mixed. In this section we did a careful ablation study over this (Table~\ref{tab:hybrid-ablation}).

\begin{table}[t]
\centering
\footnotesize
\caption{\textbf{Hybrid LT2 ablations.} $1.3$B / $T{=}4$ / $100$B FineWeb-Edu
tokens. \colorbox{lt2cream}{Cream} marks the best row in each group. Avg.\ is
the eight-task mean from Table~\ref{tab:downstream}.}
\label{tab:hybrid-ablation}
\setlength{\tabcolsep}{4pt}
\renewcommand{\arraystretch}{1.0}
\begin{tabular}{l l l c c}
\toprule
\textbf{Configuration} & \textbf{Full$:$GDN} & \textbf{Pattern / Schedule} & \textbf{PPL}~$(\downarrow)$ & \textbf{Avg.}~$(\uparrow)$ \\
\midrule
\rowcolor{lt2navy!8}\multicolumn{5}{l}{\textit{\textbf{\textcolor{lt2navy}{$\blacktriangleright$ (1) Hybrid ratio} (depth-interleaved)}}} \\
Looped Transformer (ref)
                  & $1{:}0$  & ---                  & 9.87  & 59.27 \\
Hybrid $1{:}1$    & $1{:}1$  & interleave           & 9.41  & 60.92 \\
\hero
Hybrid $1{:}4$ (\textbf{default})
                  & $1{:}4$  & interleave           & 9.31  & 61.39 \\
Hybrid $1{:}6$    & $1{:}6$  & interleave           & 9.36  & 61.07 \\
Hybrid $1{:}12$   & $1{:}12$ & interleave           & 9.74  & 59.51 \\
Looped GDN        & $0{:}1$  & ---                  & 10.02 & 58.42 \\
\addlinespace[2pt]
\rowcolor{lt2navy!8}\multicolumn{5}{l}{\textit{\textbf{\textcolor{lt2navy}{$\blacktriangleright$ (2) Hybrid pattern} (ratio fixed at $1{:}4$, depth-level)}}} \\
\hero
Bookend           & $1{:}4$  & Full at top \& bottom, GDN in middle
                                                    & 9.27  & 61.52 \\
Interleave (\textbf{default})
                  & $1{:}4$  & every 5th layer is Full
                                                    & 9.31  & 61.39 \\
Front-loaded      & $1{:}4$  & all Full layers at the bottom of the stack
                                                    & 9.45  & 60.61 \\
Back-loaded       & $1{:}4$  & all Full layers at the top of the stack
                                                    & 9.53  & 60.43 \\
\addlinespace[2pt]
\rowcolor{lt2navy!8}\multicolumn{5}{l}{\textit{\textbf{\textcolor{lt2navy}{$\blacktriangleright$ (3) Hybridization level} (matched parameters)}}} \\
\hero
Random sample $+$ majority vote ($K{=}5$)
                  & $1{:}4$  & resample $1/5$ Full per step; vote at eval
                                                    & 9.26  & 61.55 \\
Depth-level (\textbf{default})
                  & $1{:}4$  & per-layer Full / GDN interleave
                                                    & 9.31  & 61.39 \\
Loop-level coarse$\to$fine
                  & ---      & Full~$\to$~SWA-512~$\to$~SWA-256~$\to$~SWA-128
                                                    & 9.36  & 60.71 \\
Loop-level fine$\to$coarse
                  & ---      & SWA-128~$\to$~SWA-256~$\to$~SWA-512~$\to$~Full
                                                    & 9.42  & 61.10 \\
\bottomrule
\end{tabular}
\vspace{-0.3em}
\end{table}

\paragraph{Ratio: clean inverse-U with the optimum at $1{:}4$.}
Sweeping the Full$:$GDN ratio between Looped Transformer ($1{:}0$) and
Looped GDN ($0{:}1$), the interior traces a clean inverse-U with $1{:}4$
on top. Too much attention crowds out the recurrent regularization
documented in Section~\ref{sec:stability}; too little starves the loop of
precise retrieval. $1{:}4$ is the smallest amount of attention that still
recovers full retrieval quality, matching standard hybrid Transformer
baselines~\citep{lahoti2026mamba3improvedsequencemodeling}.

\paragraph{Pattern: spreading beats concentrating.}
At fixed $1{:}4$, bookend (Full at top and bottom, GDN in the middle)
slightly edges out the uniform interleave, hinting at a small benefit
from attention at both the input encoding and the final read-out.
Concentrating the attention layers at one end --- front-loaded or
back-loaded --- loses more than $0.7$ points of average accuracy. Maybe the
takeaway is that any reasonable spread along depth is much better than
any concentration. 

\paragraph{Level: across-iteration heterogeneity does not help.}
We try three loop-level schedules in place of depth-level mixing:
\emph{coarse-to-fine} (Full$\to$SWA-512$\to$SWA-256$\to$SWA-128),
\emph{fine-to-coarse} (the reverse), and a stochastic baseline that
resamples a $1{:}4$ depth-level hybrid every step and majority-votes
$K{=}5$ samples at eval. Coarse-to-fine wins on PPL but loses on
downstream, it narrows the receptive field on the final iteration
over-fits local statistics. Fine-to-coarse inverts the trade. Random-vote
is the best overall but at $5{\times}$ inference compute,
which is hard to justify against the simple fixed interleave.

\subsection{Ablation: attention sinks and the SDPA output gate}
\label{sec:abl-output-gate}

A natural concern with weight-shared loops is that pathologies of the
underlying attention block (in particular, the \emph{attention sink}~\citep{xiao2024efficientstreaminglanguagemodels}) where a
small set of tokens absorb a disproportionate share of softmax
mass~\citep{sun2024massiveactivationslargelanguage} may compound
across loop iterations: the same softmax block is re-applied to a residual
stream that already carries the sink from the previous pass.
Gated Attention~\citep{qiu2025gatedattentionlargelanguage} show that a head-specific sigmoid gate after Scaled
Dot-Product Attention (SDPA) eliminates the sink in standard transformers. We
ask the same question for our looped models and adopt the same fix, applied inside the looped block so $W_\theta$
is reused on every iteration. We add this gate to the three
LT2 variants, keep the FFN width matched in parameter count, and re-train at
$1.3$B/$T{=}4$ on $100$B FineWeb-Edu tokens.
\begin{figure}[t]
    \centering
    \includegraphics[width=0.98\textwidth]{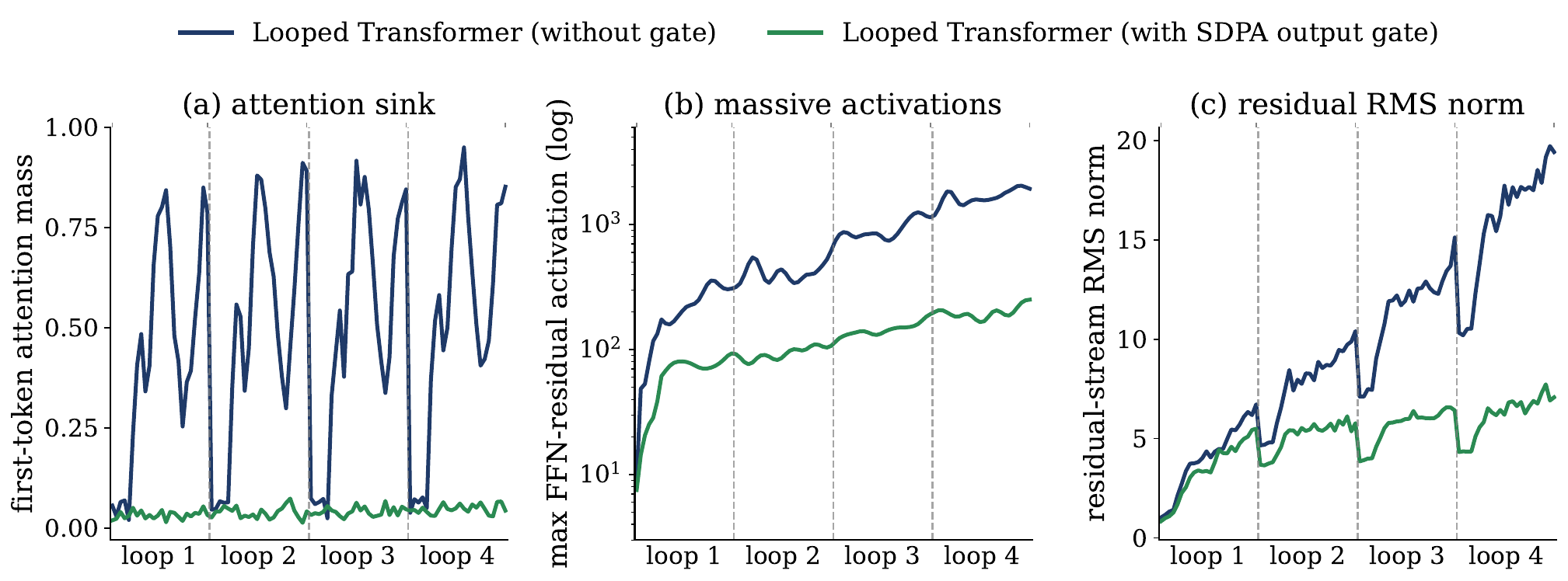}
    \vspace{-0.6em}
    \caption{\textbf{Unrolled diagnostics for the Looped Transformer
    ($T{=}4$, 24 layers).} The x-axis runs through the unrolled computation,
    with dashed lines marking loop boundaries.
    \textbf{(a)} First-token attention mass forms a sawtooth that intensifies
    each loop --- the sink is re-injected rather than reset.
    \textbf{(b)} Max FFN-residual activation follows the same compounding
    pattern (log scale).
    \textbf{(c)} Residual-stream RMS norm grows with both within-loop depth
    and across loop iterations.
    The SDPA output gate flattens (a)/(b) and substantially mitigates ---
    though does not eliminate --- the across-loop growth in (c).}
    \label{fig:loop-sink-massact}
    \vspace{-1.3em}
\end{figure}

\begin{wraptable}{r}{0.50\textwidth}
\vspace{-1.3em}
\centering
\footnotesize
\setlength{\tabcolsep}{4pt}
\renewcommand{\arraystretch}{1.0}
\caption{\textbf{Effect of the SDPA output gate on the three softmax-containing
LT2 variants.} 1.3B / $T{=}4$ / 100B tokens. We report the mean over the
eight zero-shot benchmarks of Table~\ref{tab:downstream}.}
\label{tab:gate-ablation}
\begin{tabular}{l c c c}
\toprule
\textbf{Model} & \textbf{Gate} & \textbf{PPL}~$(\downarrow)$ & \textbf{Avg.}~$(\uparrow)$ \\
\midrule
\multirow{3}{*}{Looped Transformer}
    & ---       & 9.87  & 59.27 \\
    & \yes      & \textbf{9.39}  & \textbf{60.70} \\
    & $\Delta$  & $-0.48$ & $+1.43$ \\
\midrule
\multirow{3}{*}{LT2-Hybrid (Full+GDN)}
    & ---       & 9.31  & 61.39 \\
    & \yes      & \textbf{9.03}  & \textbf{62.33} \\
    & $\Delta$  & $-0.28$ & $+0.94$ \\
\midrule
\multirow{3}{*}{LT2-Hybrid (GDN+DSA)}
    & ---       & 9.72  & 59.23 \\
    & \yes      & \textbf{9.53}  & \textbf{59.96} \\
    & $\Delta$  & $-0.19$ & $+0.73$ \\
\bottomrule
\end{tabular}
\vspace{-1.3em}
\end{wraptable}
\paragraph{The sink is real and compounds across loops.} 
Figure~\ref{fig:loop-sink-massact}
unrolls the Looped Transformer along the trajectory
$(\text{loop }1,\text{layer }1)\!\to\!(\text{loop }1,\text{layer }24)\!\to\!(\text{loop }2,\text{layer }1)\!\to\!\cdots$. The first-token attention mass traces a sawtooth that
\emph{intensifies across loops}: the sink learned in loop~$t$ is re-injected
into loop~$t{+}1$ rather than reset, so each successive iteration starts
already biased toward the sink. The maximum residual activation
follows the same compounding pattern, consistent with the
established literature~\citep{sun2024massiveactivationslargelanguage}. And the residual-stream RMS norm grows along both within-loop depth and across loop iterations, reaching
${\sim}20$ times by the end of loop~4.

Overall, attention sinks and massive activations are not
artifacts of standard transformers alone. weight-shared loops mildly amplify
them by re-applying the same softmax to a residual that already carries the
sink. A single head-specific sigmoid gate inside the loop suppresses the
compounding and yields a small but consistent improvement on every
softmax-containing LT2 variant. We recommend adopting this gate in all such variants in looped setting.

\subsection{Training stability}
\label{sec:stability}

A practical concern with looped models is that repeatedly applying the same block can amplify activations and destabilize optimization. We track the language-modeling loss and global gradient norm throughout pre-training and find that the choice of mixer inside the loop has a pronounced effect on stability.

\paragraph{Gating and the delta rule keep the linear loop bounded.}
\begin{wrapfigure}{r}{0.6\textwidth}
    \vspace{-1.0em}
    \centering
    \includegraphics[width=0.6\textwidth]{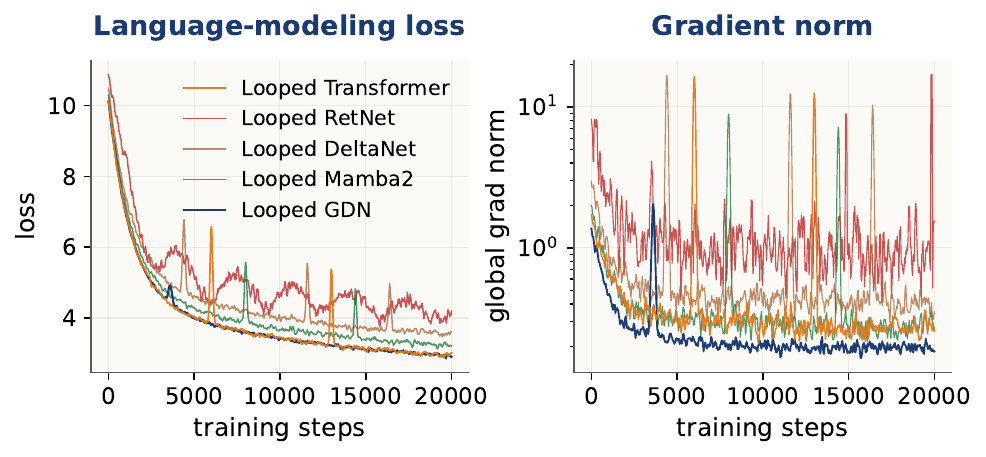}
    \caption{Looped GDN trains with the smoothest loss and the smallest gradient norms across all linear and full-attention variants; Looped RetNet, which lacks both data-dependent gating and a delta rule, diverges.}
    \label{fig:linear-stability}
    \vspace{-0.5em}
\end{wrapfigure}
Figure~\ref{fig:linear-stability} compares the looped Transformer against four subquadratic mixers. Looped RetNet exhibits persistently large gradient norms and frequent spikes throughout training, consistent with its divergence in Table~\ref{tab:downstream}. Looped DeltaNet and Looped Mamba2 are noticeably better but still display occasional spikes that propagate into the loss. In contrast, Looped GDN tracks below the Looped Transformer in gradient norm for essentially the entire run and produces the smoothest loss curve of any variant. Two ingredients appear to matter: a data-dependent gate, which lets the recurrence forget stale state instead of letting it accumulate across iterations, and the delta rule, which bounds updates to the recurrent memory. Mixers that have only one of the two (Mamba2 has gating without the delta rule; DeltaNet has the delta rule with weaker gating) are stable but visibly noisier than GDN, while RetNet, which has neither, is unstable.

\paragraph{Sparse attention is stable but slightly less capable.}
\begin{wrapfigure}{r}{0.6\textwidth}
    \vspace{-1.0em}
    \centering
    \includegraphics[width=0.6\textwidth]{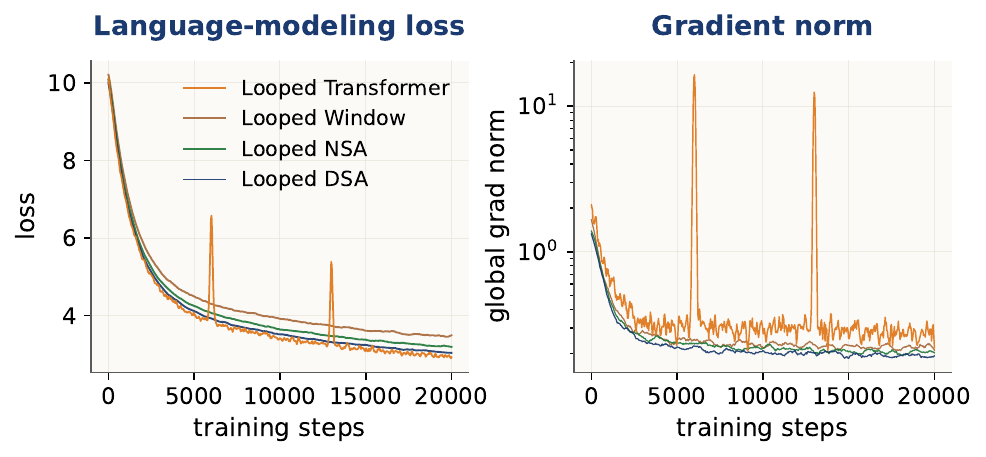}
    \caption{Sparse looped variants train without the spikes seen in the full-attention loop, but reach a slightly higher final loss than the Looped Transformer.}
    \label{fig:sparse-stability}
    \vspace{-0.5em}
\end{wrapfigure}
Figure~\ref{fig:sparse-stability} reports the same diagnostics for sparse-attention loops. All three sparse variants (Window, NSA, DSA) train smoothly: their gradient norms sit at or below the Looped Transformer for the entire run, and none of them shows the sharp spikes that occasionally appear in the full-attention loop around the middle of training. The price is a small but consistent gap in language-modeling loss --- restricting each iteration to a sparse receptive field caps the per-loop computation and slows convergence relative to dense attention. Among the sparse choices, Looped DSA is the strongest, which is why we adopt it as the sparse component of LT2 in the remainder of the paper.

\paragraph{Hybrid mixers combine stability and capability.}
\begin{wrapfigure}{r}{0.6\textwidth}
    \vspace{-1.0em}
    \centering
    \includegraphics[width=0.6\textwidth]{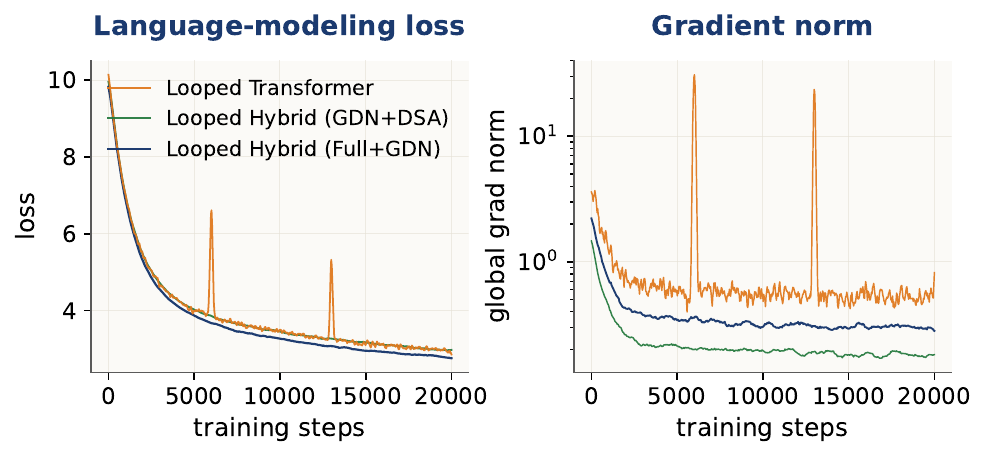}
    \caption{Both hybrid variants match or beat the Looped Transformer in loss while producing smaller and smoother gradient norms throughout training.}
    \label{fig:hybrid-stability}
    \vspace{-0.5em}
\end{wrapfigure}
Figure~\ref{fig:hybrid-stability} shows that the two hybrid configurations inherit the best of both worlds. Looped Hybrid (GDN+DSA) and Looped Hybrid (Full+GDN) both track the Looped Transformer in loss from the very beginning of training and pull slightly ahead by the end, while their gradient norms remain consistently smaller and free of the spikes that the full-attention loop occasionally produces. Pairing a recurrent mixer with either sparse or dense attention thus appears to regularize the loop: the linear branch keeps the gradient norm bounded across iterations, while the attention branch supplies the precise retrieval that pure linear models lack.

Across all three comparisons the same picture emerges. Mixers with data-dependent gating and a delta rule (GDN, and the hybrids that include it) train more stably than vanilla full attention under looping, and sparse attention, while less expressive, never destabilizes. This motivates the two LT2 instantiations used in the rest of the paper: \emph{LT2-sparse} (Looped Hybrid with DSA) when stability is the priority, and \emph{LT2-linear} (Looped Hybrid with GDN) when capability is.

\subsection{Synthetic Tasks: State-tracking + Recall}
\label{sec:recall-state}

State tracking and long-range retrieval are usually treated as opposite
stress tests: state tracking favors recurrent depth, while retrieval
favors precise attention to a long history. The first synthetic
experiment we use to probe LT2 puts both pressures on the same task. We
follow the \emph{state-based recall} construction of Olmo-hybrid~\citep{merrill2026olmohybridtheorypractice}.
\begin{center}
\begin{tcolorbox}[colback=white, colframe=lt2navy!60, arc=2pt, boxrule=0.4pt,
    left=6pt, right=6pt, top=4pt, bottom=4pt, width=0.96\linewidth,
    fontupper=\ttfamily\footnotesize\linespread{1.05}\selectfont]
\textcolor{lt2navy}{\textbf{bits}} = [\,\textcolor{lt2good}{\textbf{1}},\textcolor{lt2bad}{\textbf{0}},\textcolor{lt2good}{\textbf{1}},\textcolor{lt2good}{\textbf{1}},\textcolor{lt2bad}{\textbf{0}},\textcolor{lt2bad}{\textbf{0}},\textcolor{lt2good}{\textbf{1}},\textcolor{lt2bad}{\textbf{0}},\textcolor{lt2good}{\textbf{1}},\textcolor{lt2good}{\textbf{1}},\textcolor{lt2bad}{\textbf{0}},\textcolor{lt2good}{\textbf{1}},\textcolor{lt2bad}{\textbf{0}},\textcolor{lt2good}{\textbf{1}},\textcolor{lt2good}{\textbf{1}},\textcolor{lt2bad}{\textbf{0}},\textcolor{lt2bad}{\textbf{0}},\textcolor{lt2good}{\textbf{1}},\textcolor{lt2bad}{\textbf{0}},\textcolor{lt2good}{\textbf{1}},\textcolor{lt2bad}{\textbf{0}},\textcolor{lt2good}{\textbf{1}},\textcolor{lt2good}{\textbf{1}},\textcolor{lt2bad}{\textbf{0}},\textcolor{lt2good}{\textbf{1}},\textcolor{lt2bad}{\textbf{0}},\textcolor{lt2good}{\textbf{1}},\textcolor{lt2good}{\textbf{1}},\textcolor{lt2bad}{\textbf{0}},\textcolor{lt2bad}{\textbf{0}},\textcolor{lt2good}{\textbf{1}},\textcolor{lt2bad}{\textbf{0}}\,]\hfill\textcolor{lt2russet}{\itshape\rmfamily\footnotesize\textleftarrow\ recall target}\\
a, b, c, d, e = 11, 27, 4, 19, 30\hfill\textcolor{lt2russet}{\itshape\rmfamily\footnotesize\textleftarrow\ pointer state: (a,b,c,d,e)=(11,27,4,19,30)}\\
a, c = c, a;\hfill\textcolor{lt2russet}{\itshape\rmfamily\footnotesize swap a$\leftrightarrow$c\ \ \textrightarrow\ \ a=4,\ c=11}\\
b, d = d, b;\hfill\textcolor{lt2russet}{\itshape\rmfamily\footnotesize swap b$\leftrightarrow$d\ \ \textrightarrow\ \ b=19,\ d=27}\\
a, e = e, a;\hfill\textcolor{lt2russet}{\itshape\rmfamily\footnotesize swap a$\leftrightarrow$e\ \ \textrightarrow\ \ a=30,\ e=4}\\
c, d = d, c;\hfill\textcolor{lt2russet}{\itshape\rmfamily\footnotesize swap c$\leftrightarrow$d\ \ \textrightarrow\ \ c=27,\ d=11}\\
d, e = e, d;\hfill\textcolor{lt2russet}{\itshape\rmfamily\footnotesize swap d$\leftrightarrow$e\ \ \textrightarrow\ \ d=4,\ e=11}\\
a, b = b, a;\hfill\textcolor{lt2russet}{\itshape\rmfamily\footnotesize swap a$\leftrightarrow$b\ \ \textrightarrow\ \ a=19,\ b=30}\\
b, c = c, b;\hfill\textcolor{lt2russet}{\itshape\rmfamily\footnotesize swap b$\leftrightarrow$c\ \ \textrightarrow\ \ b=27,\ c=30}\\
a, d = d, a;\hfill\textcolor{lt2russet}{\itshape\rmfamily\footnotesize swap a$\leftrightarrow$d\ \ \textrightarrow\ \ a=4,\ d=19}\\
\colorbox{lt2cream}{\hspace{1pt}assert bits[a] == \colorbox{lt2good!18}{\textcolor{lt2good}{\textbf{?}}}\hspace{1pt}}\hfill\textcolor{lt2russet}{\itshape\rmfamily\footnotesize final query: a=4, so target = bits[4] = \textbf{0}}
\end{tcolorbox}
\end{center}
\paragraph{Task and setup.} Each example is a short Python-like program
(see below for $m{=}32$, $n{=}8$): a bit array \texttt{bits} of length $m$
is written, five variables \texttt{a}--\texttt{e} are bound to five
distinct indices in $[0, m)$, then $n$ swap lines are emitted with an
\texttt{assert bits[a] == ?} after every swap. The model is supervised on
each \texttt{?} token. The task combines \emph{long-range recall} (fetch
\texttt{bits[$\cdot$]} from $\Theta(m)$ tokens earlier, hard for
compressive RNNs) with \emph{state tracking} (apply the running sequence
of transpositions to the pointer, hard for fixed-depth Transformers in
$\mathrm{TC}^0$). We tie $n{=}m$ and grow them together along the
curriculum $\{8,16,32,64,128,256\}$; a model advances once eval accuracy
reaches $0.90$ within a $100$k-step budget. The headline metric is
$n_{\max}$, the largest $n{=}m$ solved. All models share a $4$-layer,
$256$-wide, $4$-head backbone (RoPE); loop variants share weights across
$T$ iterations of this backbone, $T{=}1$ is the standard non-loop model.
We train with AdamW (peak LR $3{\times}10^{-4}$, batch $32$).

\begin{wrapfigure}{r}{0.55\textwidth}
    \vspace{-1.0em}
    \centering
    \includegraphics[width=0.55\textwidth]{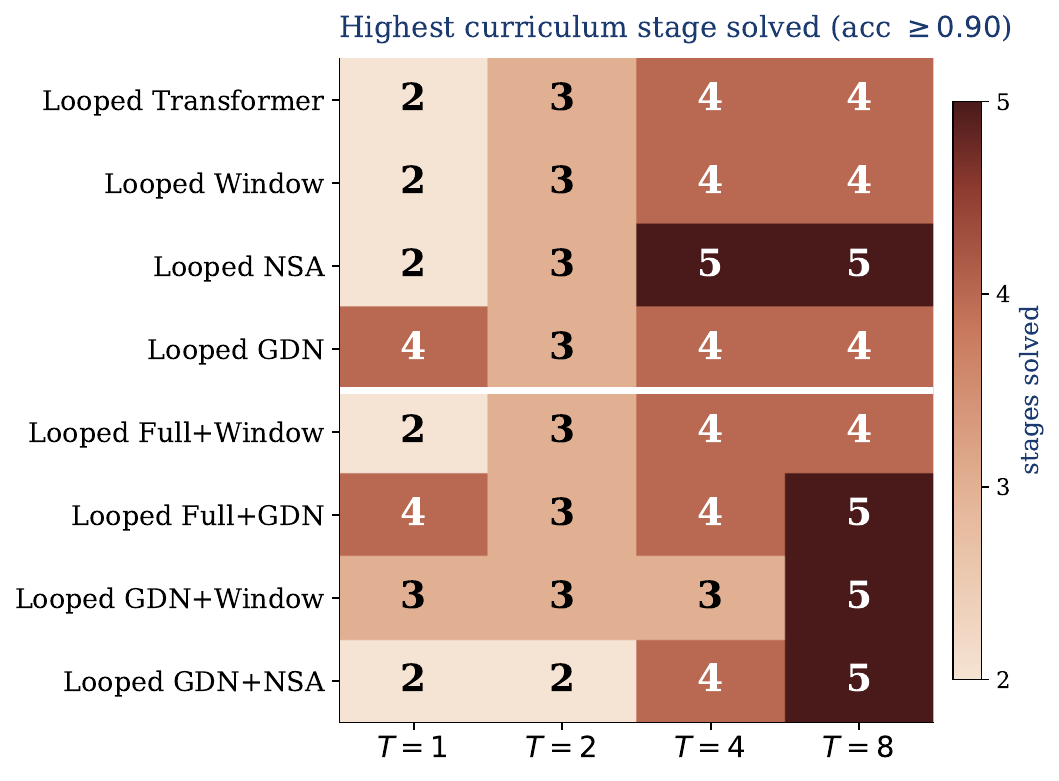}
    \caption{Highest curriculum stage $s$ solved wrt. loop count $T$. Pure mixers above the white line, hybrids
        below.}
    \label{fig:state-recall-results}
    \vspace{-0.8em}
\end{wrapfigure}
\paragraph{Effect of looping.} Figure~\ref{fig:state-recall-results}
reports the highest stage solved for every (architecture, $T$). The
striking pattern is that \emph{looping helps subquadratic mixers more
than it helps full attention.} Looped Transformer and Looped Full+Window
plateau at stage $4$ ($n_{\max}{=}64$) and never reach stage $5$ at any
$T$. By contrast, three subquadratic variants --- Looped NSA, Looped
GDN+Window, and Looped GDN+NSA --- all reach stage $5$
($n_{\max}{=}128$); Looped Full+GDN does too, but only because half its
block is already linear-time GDN. Looped GDN+Window is the most
dramatic: stage $3$ at $T{\leq}4$, stage $5$ at $T{=}8$. 

\paragraph{Comparison across architectures.} The reference
point is the Looped Transformer. Relative to it, several subquadratic
mixers gain \emph{both} expressive power and recall on this joint task.
Looped Transformer caps at stage $4$ at all $T$, but Looped NSA, Looped
GDN+Window, and Looped GDN+NSA all reach stage $5$ --- a doubling of
$n_{\max}$ over the global-attention baseline at the same parameter
budget. Among these, Looped GDN+NSA and GDN+Window are the fully linear-time
models to solve it, consistent with the great language modeling performance discussed above.

\subsection{Realistic recall and long-context retrieval}
\label{sec:long-context}

 We now turn to realistic long-context recall, where the model must retrieve specific facts from natural text far longer than fits comfortably into a recurrent state. We follow the evaluation protocol of Mamba-3~\citep{lahoti2026mamba3improvedsequencemodeling}. All models are at the $1.3$B scale. pure models stack a single mixer family throughout, while hybrids interleave that mixer with full attention in a fixed $4{:}1$ ratio, and looped variants share weights across $T{=}4$ iterations of the corresponding non-looped backbone at matched parameter count. We evaluate two complementary suites: knowledge-style recall on SWDE~\citep{arora2025simplelinearattentionlanguage}, SQuAD~\citep{rajpurkar2016squad100000questionsmachine}, FDA~\citep{arora2025simplelinearattentionlanguage}, TriviaQA~\citep{joshi2017triviaqalargescaledistantly}, Natural Questions~\citep{kwiatkowski-etal-2019-natural}, and DROP~\citep{dua2019dropreadingcomprehensionbenchmark} at $2048$ tokens, and precise needle-in-a-haystack retrieval (NIAH-Single-1/2/3)~\citep{hsieh2024rulerwhatsrealcontext} at $1024$, $2048$, and $4096$ tokens, where the $4096$ column is the most informative since all models were pre-trained at $2048$ and must extrapolate.

Three findings stand out (Table~\ref{tab:long-context}). Looping consistently improves the underlying mixer at matched parameter count: Looped Transformer, Looped GDN, and Looped Mamba-2 each gain roughly $2$--$4$ points on average over their non-looped counterparts on the knowledge-style suite, with task-level fluctuations in either direction (some columns such as TQA are already saturated, while FDA and NQ benefit more from extra iterations), and they retain the qualitative NIAH behavior of their base versions -- recurrent backbones extrapolate gracefully to $4096$ while dense-attention ones do not.

Looped Hybrid (GDN+DSA) tracks the Looped Transformer closely on the knowledge suite despite containing no quadratic component, and additionally extrapolates substantially better at NIAH-4096.
\begin{table*}[ht]
\centering
\caption{Long-context evaluation at $1.3$B parameters. Knowledge-style benchmarks (SWDE--DROP) are evaluated at $2048$ tokens; NIAH-Single-1/2/3 at $1024$, $2048$, and $4096$ tokens (models pre-trained at $2048$ must extrapolate to $4096$). Pure models use a single mixer; hybrid models interleave with full attention in a $1{:}1$ ratio. Looped variants share weights across $T{=}4$ iterations of the corresponding non-looped backbone. \textbf{Bold} marks the best result per column; \underline{underline} marks the second best.}
\label{tab:long-context}
\renewcommand{\arraystretch}{1.15}
\setlength{\tabcolsep}{4pt}
\resizebox{0.95\linewidth}{!}{%
\begin{tabular}{@{}l cccccc ccc ccc ccc@{}}
\toprule
\rowcolor{lt2navy!90}
\textcolor{white}{\textbf{Model (1.5B)}} &
\textcolor{white}{SWDE} & \textcolor{white}{SQD.} & \textcolor{white}{FDA} & \textcolor{white}{TQA} & \textcolor{white}{NQ} & \textcolor{white}{DROP} &
\multicolumn{3}{c}{\textcolor{white}{NIAH-Single-1}} &
\multicolumn{3}{c}{\textcolor{white}{NIAH-Single-2}} &
\multicolumn{3}{c}{\textcolor{white}{NIAH-Single-3}} \\
\rowcolor{lt2navy!90}
\textcolor{white}{Context Length} &
\multicolumn{6}{c}{\textcolor{white}{2048}} &
\textcolor{white}{1024} & \textcolor{white}{2048} & \textcolor{white}{4096} &
\textcolor{white}{1024} & \textcolor{white}{2048} & \textcolor{white}{4096} &
\textcolor{white}{1024} & \textcolor{white}{2048} & \textcolor{white}{4096} \\
\midrule

\rowcolor{lt2bandblue}
Transformer
  & 48.9 & 46.6 & 58.4 & 67.5 & 31.7 & 26.4
  & 100.0 & 100.0 & 0.0
  & 92.2 & 100.0 & 0.0
  & 98.6 & 99.4 & 0.0 \\

\rowcolor{lt2bandred}
GDN
  & 32.7 & 40.0 & 28.3 & 63.5 & 25.7 & 24.5
  & 100.0 & 100.0 & 99.8
  & 100.0 & 93.8 & 49.8
  & 83.8 & 68.4 & 34.2 \\
\rowcolor{lt2bandred}
Mamba-2
  & 30.7 & 39.1 & 23.7 & 64.3 & 25.1 & 28.5
  & 100.0 & 99.6 & 62.0
  & 100.0 & 53.8 & 11.8
  & 95.8 & 87.4 & 13.4 \\

\midrule

\rowcolor{lt2bandred}
Looped Transformer
  & \underline{52.8} & \textbf{49.4} & \underline{61.7} & \textbf{68.2} & \underline{33.6} & 28.1
  & 100.0 & 100.0 & 0.0
  & 94.6 & 100.0 & 0.0
  & 99.2 & \textbf{99.8} & 0.0 \\
\rowcolor{lt2bandred}
Looped GDN
  & 34.9 & 41.8 & 30.6 & 64.7 & 27.0 & 25.9
  & 100.0 & 100.0 & \textbf{99.8}
  & 100.0 & \underline{96.4} & 53.2
  & 85.6 & 71.0 & 35.8 \\
\rowcolor{lt2bandred}
Looped Mamba-2
  & 33.9 & 40.5 & 25.8 & 65.1 & 26.8 & \underline{29.7}
  & 100.0 & 100.0 & 65.7
  & 100.0 & 57.1 & 13.5
  & 96.2 & 88.1 & 16.2 \\

\midrule

\rowcolor{lt2bandgreen}
Looped Hybrid (GDN+DSA)
  & 51.6 & 48.0 & 60.4 & 66.9 & 33.0 & 28.4
  & 100.0 & 100.0 & 91.4
  & 100.0 & \textbf{100.0} & \underline{77.6}
  & \textbf{100.0} & \underline{99.6} & \underline{60.3} \\
\rowcolor{lt2cream}
\textbf{Looped Hybrid (Full+GDN)}
  & \textbf{53.1} & \underline{48.9} & \textbf{62.0} & \underline{67.8} & \textbf{34.0} & \textbf{30.2}
  & 100.0 & 100.0 & \underline{93.5}
  & 100.0 & \textbf{100.0} & \textbf{81.0}
  & \underline{99.8} & \textbf{99.8} & \textbf{63.7} \\

\bottomrule
\end{tabular}%
}
\end{table*}
Looped Hybrid (Full+GDN) is the strongest configuration overall, improving over the Looped Transformer on the average of the knowledge-style benchmarks while degrading more gracefully than the Looped Transformer at NIAH-4096 thanks to its GDN branch. As before, individual cells fluctuate across this broad suite, but the ordering on the aggregate is stable and matches the picture from Section~\ref{sec:lm}: the linear--sparse loop (LT2-sparse) recovers full-attention quality, and the full--linear loop (LT2-linear) extends the Pareto frontier.

\section{Distilling Looped Transformers into LT2}
\label{sec:distillation}

A natural follow-up to the from-scratch results above is whether LT2
can also be reached \emph{post-hoc}, by replacing most of a pre-trained
Looped Transformer's quadratic attention with linear-time attention variants.
We extend the hybrid-distillation recipe of non-looped models~\citep{li2025distillinghybridattentionmodels} to the looped setting
with the full attention ratio of $25\%$. The overall results is shown in Figure~\ref{fig:distill}.

\subsection{Algorithm}
\label{sec:distill-algo}

We follow the two-stage RADLADS-style pipeline~\citep{li2025distillinghybridattentionmodels}. The teacher $f_{\theta_\mathcal{T}}$ is Ouro-1.4B~\citep{zhu2025scalinglatentreasoninglooped} which is a standard looped transformer. The student
$f_{\theta_\mathcal{S}}$ shares the teacher's embedding, FFN, and norm
parameters but uses GDN as token mixer outside a small kept-attention
set $\mathcal{F}\!\subseteq\!\{1,\dots,N\}$.
At loop iteration $\tau\!\in\!\{1,\dots,T\}$ each network emits logits
$z^{(\tau)}_{\mathcal{T}/\mathcal{S}}\!\in\!\mathbb{R}^{L\times V}$, where $L$ is the sequence length and $V$ is the vocabulary size.
\begin{figure}[t]
 \vspace{-0.3cm}
  \centering
  \includegraphics[width=\linewidth]{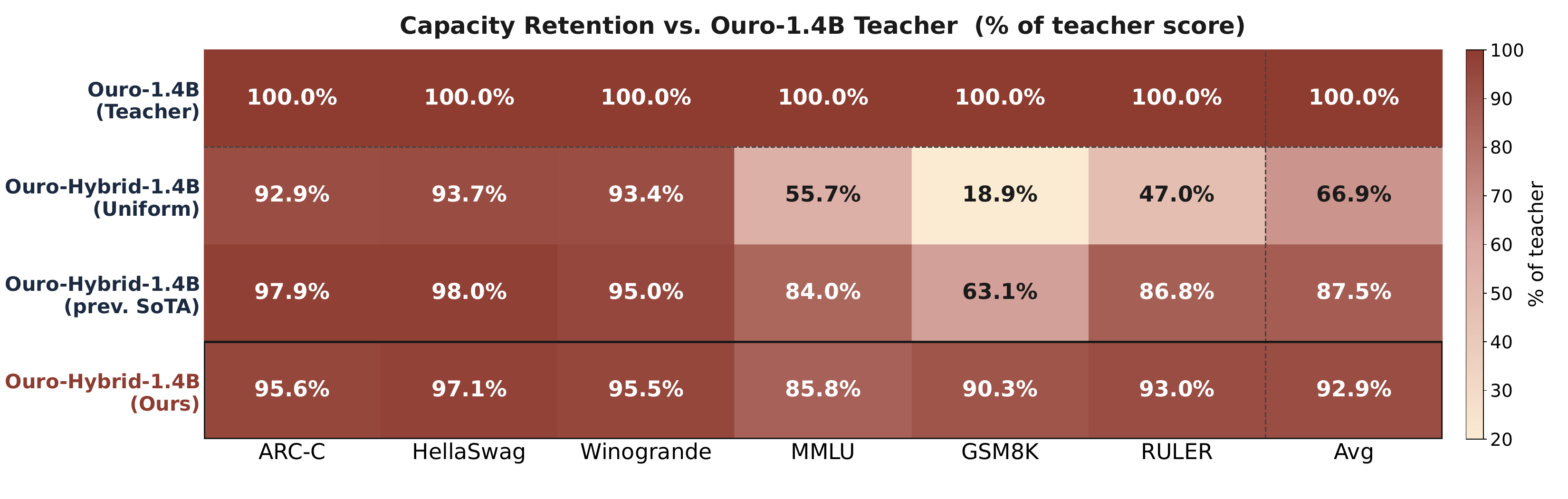}
  \vspace{-0.7cm}
  \caption{\textbf{Capability retention of distilled
    Ouro-Hybrid-1.4B variants.}
    \emph{Ouro-Hybrid (Uniform)} interleaves linear and full-attention layers in a fixed interleaved pattern, while \emph{Ouro-Hybrid (prev.\ SoTA)} selects the layers to retain full attention using the best method previously~\citep{li2025distillinghybridattentionmodels}.}
  \label{fig:distill}
  \vspace{-0.3cm}
\end{figure}

\textbf{Stage~1 (linear pre-alignment).} With $\mathcal{F}\!=\!\varnothing$, we align each GDN block to the
corresponding teacher attention output via an MSE loss on the residual
stream ($100$M tokens, length $512$). 

\textbf{Stage~2 (hybrid logit distillation).} We restore the $|\mathcal{F}|\!=\!6$ softmax layers selected by the
KL-guided selector to pick top full attention layers and
distill the teacher's logits.
The looped setting introduces a new design knob -- the
\emph{per-loop} KL schedule -- and we minimise
\begin{equation}
  \mathcal{L}_{\text{KD}}
  \;=\;
  \sum_{\tau=1}^{T} w^{(\tau)}_{t}\,
  \mathrm{KL}\!\left(
    \sigma_{\mathrm{top}\text{-}k}\!\big(z^{(\tau)}_\mathcal{T}/T_{\!\mathrm{kd}}\big)\,\Big\|\,
    \sigma_{\mathrm{top}\text{-}k}\!\big(z^{(\tau)}_\mathcal{S}/T_{\!\mathrm{kd}}\big)
  \right),
  \label{eq:kd-loss}
\end{equation}
where $\sigma_{\mathrm{top}\text{-}k}$ renormalises the softmax over
the teacher's top-$k$ tokens and
$\mathbf{w}_t\!\in\!\Delta^{T-1}$ controls how much supervision each
loop receives at step $t$. We progressively warm up the loop-level supervision, set uniform weights to supervise equally per loop for half of the training steps, and then switch to final-output supervision only ($600$M tokens, length $4096$).
\begin{figure}[b]
\vspace{-0.5cm}
  \centering
  \includegraphics[width=\linewidth]{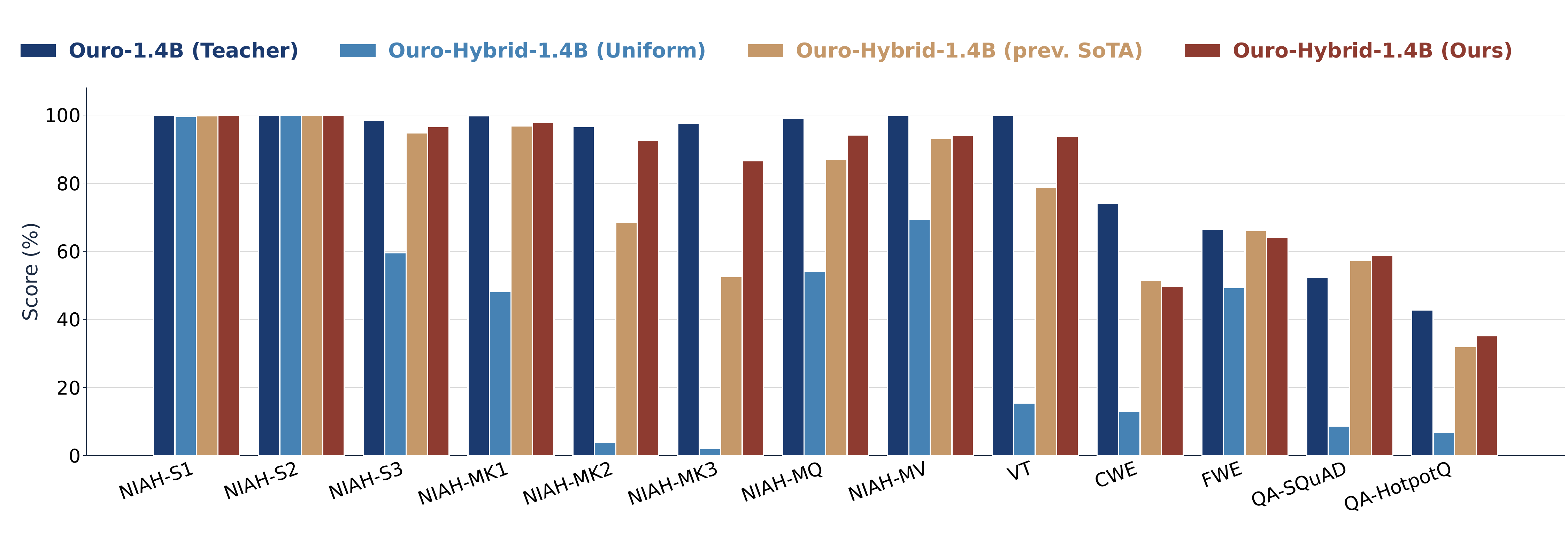}
  \vspace{-0.7cm}
 \caption{Ruler subtask performance for different distillation models. The key task differences lie in multi-key retrieval, which benefits more from per-loop supervision.}
  \label{fig:ruler-subtask}
\vspace{-0.3cm}
\end{figure}

\textbf{Stage 3 (Long-context continuation).} We then
extend Stage 2 with a continuation phase on $35$K OpenThoughts-v3~\citep{guha2025openthoughtsdatarecipesreasoning} with 32k sequence length, using the same KD loss in
Eq.~\eqref{eq:kd-loss} and a constant LR ($600$M tokens, length $32768$).

\subsection{Results}

Distilling pre-trained full-attention models into linear-time variants remains non-trivial. We find that the primary driver of performance on general and mathematical benchmarks is the composition of the distillation data. Stages 1 and 2 utilize general datasets like DCLM~\citep{li2025datacomplmsearchgenerationtraining}, resulting in high scores on commonsense downstream tasks. Integrating reasoning-specific data like OpenThoughts in Stage 3 significantly narrows the gap in mathematical reasoning. Two factors are important from an algorithmic perspective. As shown in Figure~\ref{fig:ruler-subtask}, progressive length expansion is essential for maintaining long-context performance. Furthermore, per-loop supervision provides a more stable gradient signal than supervising the final loop alone. We recommend this multi-stage recipe for researchers distilling looped architectures.

Finally, we compare our distilled model against industry-level small language models (Figure~\ref{fig:main}), demonstrating highly competitive performance. We will release the full model checkpoints to the public to foster further research into efficient, high-capability small models.

\section{Related Work}

\paragraph{Looped Transformers and the path to scalable recursion.}
Universal Transformers~\citep{dehghani2019universaltransformers} reintroduced depth-wise recurrence to the Transformer by tying the weights of every layer~\cite{bai2019deepequilibriummodels} and iterating the same block for a fixed or adaptive number of steps. Early follow-ups showed that this simple inductive bias improves systematic generalization on compositional benchmarks~\citep{csordas-etal-2021-devil} and that, more broadly, parameter sharing across layers is a viable design choice rather than a curiosity~\citep{takase-kiyono-2023-lessons}. On the theoretical side, looped architectures are Turing-complete under mild assumptions~\citep{perez2019turingcompletenessmodernneural} and can be programmed to emulate iterative algorithms such as multi-step gradient descent~\citep{giannou2023looped,yanglooped,gatmiry2024can,gatmiryrole,fanlooped}, which has been formalized as a ``latent thought'' view of looping where each pass refines an internal computation~\citep{saunshi2025reasoninglatentthoughtspower,gong2026makesloopedtransformersperform,blayney2026mechanisticanalysisloopedreasoning,chen2026loopbridgeloopedtransformers}. This perspective has driven a recent wave of recursion-centric reasoning models, including HRM~\citep{wang2025hierarchicalreasoningmodel}, TRM~\citep{jolicoeurmartineau2025morerecursivereasoningtiny}, the Universal Reasoning Model~\citep{gao2025universalreasoningmodel}, and looped language models trained at scale~\citep{zhu2025scalinglatentreasoninglooped}. The central obstacle, however, is scalability: looping the same block $T$ times multiplies compute by $T$ without adding parameters, so naive Universal Transformers underperform standard Transformers under matched FLOPs~\citep{tay-etal-2023-scaling,prairie2026parcaescalinglawsstable}. Three lines of work directly target this efficiency gap. The first injects sparsity and mixture-of-experts into the shared block so that capacity grows without proportional compute, as in the Sparse Universal Transformer~\citep{tan-etal-2023-sparse}, MoEUT~\citep{csordas2024moeut}, Mixture of Universal Experts~\citep{chen2026mixtureuniversalexpertsscaling}, and parameter-efficient FFN reuse~\citep{nie2026versatileffnachievingparameterefficiency}. The second relaxes strict weight tying with low-rank deltas so that each iteration can specialize cheaply~\citep{bae2025relaxedrecursivetransformerseffective}. The third, most directly in the spirit of Adaptive Computation Time, allocates a variable number of recursion steps per token: Mixture-of-Recursions learns dynamic per-token depths in a token-level routing framework~\citep{bae2025mixture}, while elastic and depth-recurrent variants extend this idea to vision and attention-aware latent reasoning~\citep{goyal2026eltelasticloopedtransformers,knupp2026depthrecurrentattentionmixturesgiving,yu2026spiralformerloopedtransformerslearn,shu2026loopvitscalingvisualarc}. Recent work has further accelerated inference of recurrent-depth models through parallel sampling~\citep{geiping2025efficientparallelsamplersrecurrentdepth}. Our work continues this trajectory, focusing on how to make looped computation scale predictably under a fixed compute budget.

\paragraph{Subquadratic attention.}
A parallel line of research replaces softmax attention with sequence mixers whose cost is linear, or near-linear, in the sequence length. Linear attention~\citep{katharopoulos2020transformers} expresses attention as a kernel feature map and reformulates inference as a recurrent state update, which is interpreted as fast-weight programming and trace back to earlier work on associative fast weights~\citep{schlag2021lineartransformerssecretlyfast,irie2021goinglineartransformersrecurrent,ba2016usingfastweightsattend}. From this foundation a family of efficient recurrences has emerged, including RetNet~\citep{sun2023retentive}, Gated Linear Attention~\citep{yanggated}, HGRN2~\citep{qin2024hgrn2}, DeltaNet and its parallel and gated variants~\citep{yang2024parallelizing,yang2024gated}, the SSM-attention duality of Mamba-2~\citep{dao2024transformersssmsgeneralizedmodels} and its successor Mamba-3~\citep{lahoti2026mamba3improvedsequencemodeling}, and RWKV-7~\citep{peng2025rwkv7gooseexpressivedynamic}. Recent work has also addressed the limited state-tracking capacity of these recurrences through negative eigenvalues~\citep{grazzi2025unlockingstatetrackinglinearrnns} and Householder products~\citep{siems2025deltaproductimprovingstatetrackinglinear}. Because pure linear-attention models still lag softmax attention on tasks requiring exact recall~\citep{arora2025simplelinearattentionlanguage}, a complementary line interleaves linear and softmax layers in hybrid stacks such as Jamba~\citep{lieber2024jambahybridtransformermambalanguage}, Kimi Linear~\citep{kimiteam2025kimilinearexpressiveefficient}, and Olmo Hybrid~\citep{merrill2026olmohybridtheorypractice}, or distills pretrained Transformers into hybrid or linear successors~\citep{goldstein2026radladsrapidattentiondistillation,li2025distillinghybridattentionmodels}. A third strand keeps softmax attention but enforces sparsity to reduce its quadratic cost, ranging from attention sinks for streaming inference~\citep{xiao2024efficientstreaminglanguagemodels} to natively trainable block-sparse patterns~\citep{yuan2025nativesparseattentionhardwarealigned,chen2025powerattentionexponentiallyscalingreceptive,deepseekai2025deepseekv32pushingfrontieropen}. These approaches are largely orthogonal to depth-wise recursion: they reduce the cost of a single forward pass, whereas looping reuses parameters across passes. Combining the two is a natural direction, and one we explore in this work.

\section{Conclusion}
We presented LT2, a family of linear-time looped Transformers that replace the quadratic token-mixing bottleneck in looped architectures with linear, sparse, and hybrid attention mechanisms. In particular, our hybrid variants recover or exceed the quality of full-attention looped Transformers while substantially improving inference efficiency. These results suggest that efficient token mixers can make recursive depth a practical scaling axis for future language models. 

\textbf{Limitations.} Two directions remain unexplored. First, we study depth-level hybridization and simple loop-level schedules but do not investigate full loop-level hybridization, where different iterations could use distinct attention families rather than only varying masks. Second, we do not design explicit cross-loop recurrent state carry mechanisms; principled state-sharing across loops may further improve long-context modeling, memory reuse, and compute efficiency.
\bibliography{references}
\bibliographystyle{abbrv}
\clearpage
\appendix
\section{Adaptive Computation Time for Looped Transformers}
\label{app:act}

This appendix gives a description of Adaptive Computation Time (ACT) for Looped Transformers and explains why we use a fixed number of loop iterations during pre-training, as stated in Section~\ref{sec:method}.

\subsection{Per-Token, Input-Dependent Compute Allocation}
\label{app:act:adaptivity}

In the basic Looped Transformer of Eq.~\eqref{eq:lt}, every token at every input is processed through exactly $T$ loop iterations. ACT relaxes this constraint by letting the number of iterations depend on (i) the specific input sequence and (ii) the token position $t \in \{1, \dots, L\}$ within that sequence. We refer to this property as \emph{adaptivity}: the model can spend more compute on token positions whose representations are still changing substantially between iterations and stop early on positions whose representations have already stabilized. Concretely, two tokens in the same sequence may halt at different iterations $\tau_t$, and the same token may halt at different iterations across different inputs.

\subsection{Halting Probabilities and Halting Rule}
\label{app:act:halting}

Following~\citep{gao2025universalreasoningmodel}, ACT introduces a small auxiliary network---a single linear layer followed by a sigmoid in our implementation---that maps the hidden state at position $t$ and iteration $\tau$ to a scalar in $(0,1)$:
\begin{equation}
    p_t^{(\tau)} = \sigma\!\left( \mathbf{w}^\top \mathbf{h}_t^{(\tau)} + b \right) \in (0,1),
    \label{eq:halting-prob}
\end{equation}
where $\mathbf{w} \in \mathbb{R}^d$ and $b \in \mathbb{R}$ are learned parameters shared across positions and iterations. We call $p_t^{(\tau)}$ the \emph{halting probability}: it represents the model's estimate, conditioned on the current hidden state, of how likely it is that no further loop iterations are needed for position $t$.

Iterations for position $t$ accumulate until the running sum of halting probabilities crosses a threshold $1 - \epsilon$:
\begin{equation}
    \tau_t \;=\; \min\Bigl\{\, \tau \;\Big|\; \sum_{\tau' = 1}^{\tau} p_t^{(\tau')} \;\geq\; 1 - \epsilon \,\Bigr\},
    \label{eq:halting-rule}
\end{equation}
where $\epsilon$ is a small positive constant (we use $\epsilon = 0.01$, matching the value used by~\citep{dehghani2019universaltransformers}). Once position $t$ halts, its hidden state is frozen at $\mathbf{h}_t^{(\tau_t)}$ and does not participate in further updates. The threshold form $1 - \epsilon$, rather than exactly $1$, ensures that the rule can be triggered after a single iteration if $p_t^{(1)}$ is sufficiently large; without the slack $\epsilon$, at least two iterations would always be required because $p_t^{(\tau)} < 1$ by construction.

\subsection{Pondering Cost and Turing Completeness}

To prevent the model from trivially driving every $p_t^{(\tau)}$ to a small value and using the maximum allowed number of iterations on every token, ACT adds a \emph{pondering cost} to the training loss that penalizes the expected number of iterations per position. The pondering cost is weighted by a scalar hyperparameter that controls the trade-off between accuracy and compute.

The combination of (i) input-dependent halting and (ii) unbounded effective depth is what allows looped Transformers with ACT to simulate arbitrary Turing machines, as shown by~\citep{perez2019turingcompletenessmodernneural}. Without ACT, a Looped Transformer with a fixed number of iterations $T$ has bounded computational depth and is therefore not Turing complete in the same formal sense.

\subsection{Why We Use Fixed Iterations During Pre-training}

Despite its theoretical appeal, ACT introduces several practical difficulties at pre-training scale:
\begin{itemize}
    \item \textbf{Optimization instability.} The halting rule in Eq.~\eqref{eq:halting-rule} is non-differentiable, and the standard relaxation used by~\citep{dehghani2019universaltransformers} couples the gradient of the pondering cost with the gradients of the main loss in ways that can produce sudden shifts in the average number of iterations during training.
    \item \textbf{Sensitivity to the pondering weight.} Small changes in the pondering hyperparameter can move the model between two degenerate regimes: halting after a single iteration on every token, or never halting until the maximum iteration cap is reached.
    \item \textbf{Throughput loss from ragged halting.} When different positions in the same batch halt at different iterations, the implementation must either pad to the longest unhalted position (losing the compute savings ACT was meant to provide) or use specialized ragged kernels.
\end{itemize}

Ouro~\citep{zhu2025scalinglatentreasoninglooped} report these instabilities for the Ouro model family at pre-training scale, and our preliminary experiments reproduced the same behavior. We therefore use a fixed number of loop iterations $T$ throughout pre-training in the main paper, and leave a stable ACT variant for future work.

\section{Proofs for Section~\ref{sec:loop-gain}}
\label{app:loop-gain}

\subsection{\texorpdfstring{Loop $\times$ DPLR linear attention: expressivity analysis}{Loop × DPLR linear attention: expressivity analysis}}
\label{app:loop-dplr}

In this section, we analyze how unrolling a DPLR linear-attention block for $T$ iterations enriches its state-transition operator. We show three things in sequence: (i) a single block applies a rank-$1$ update to the recurrent state, while $T$ stacked blocks compose into an update of rank up to $T$; (ii) by the Cartan--Dieudonn\'e theorem, this composition is expressive enough to realize any orthogonal transformation in $\mathrm{O}(d_k)$ once $T \geq d_k$; and (iii) as a concrete consequence, a single looped layer can compute prefix products for the symmetric group $S_n$ whenever $T \geq n-1$. Throughout, the spectral norm of the operator stays bounded by $1$, so stability is preserved.

\paragraph{Setup.}
An LT$_2$-LA layer with a DPLR mixer maintains a recurrent state $\mathbf{S}_{t} \in \mathbb{R}^{d_k \times d_v}$ that evolves as
\begin{equation}
\label{eq:dplr-app}
\mathbf{S}_{t} \;=\; \mathbf{A}_{t}\,\mathbf{S}_{t-1} \;+\; \beta_{t}\,\mathbf{k}_{t}\mathbf{v}_{t}^{\top}, 
\qquad 
\mathbf{A}_{t} \;=\; \mathrm{Diag}(\boldsymbol{\alpha}_{t})\,\bigl(\mathbf{I} - \beta_{t}\mathbf{k}_{t}\mathbf{k}_{t}^{\top}\bigr),
\end{equation}
where the symbols denote the following quantities:
\begin{itemize}
    \item $\mathbf{S}_t \in \mathbb{R}^{d_k \times d_v}$ is the recurrent state at step $t$, with key dimension $d_k$ and value dimension $d_v$;
    \item $\mathbf{k}_t \in \mathbb{R}^{d_k}$ is a unit-norm key vector ($\|\mathbf{k}_t\|_2 = 1$);
    \item $\mathbf{v}_t \in \mathbb{R}^{d_v}$ is the value vector;
    \item $\beta_t \in [0, 2]$ is a scalar gain controlling the strength of the rank-$1$ update;
    \item $\boldsymbol{\alpha}_t \in [0,1]^{d_k}$ is a per-channel decay vector, applied as a diagonal matrix $\mathrm{Diag}(\boldsymbol{\alpha}_t) \in \mathbb{R}^{d_k \times d_k}$;
    \item $\mathbf{I} \in \mathbb{R}^{d_k \times d_k}$ is the identity matrix.
\end{itemize}
Geometrically, the factor $(\mathbf{I} - \beta_t \mathbf{k}_t \mathbf{k}_t^{\top})$ is a generalized Householder transformation: it shrinks the component of any vector along $\mathbf{k}_t$ (by a factor $1-\beta_t$) and leaves the orthogonal complement untouched. Multiplication by $\mathrm{Diag}(\boldsymbol{\alpha}_t)$ then applies a per-channel decay. The whole map $\mathbf{A}_t \in \mathbb{R}^{d_k \times d_k}$ is therefore a \emph{rank-$1$ perturbation of a diagonal matrix}.

When the same block is unrolled for $T$ loops at step $t$ with loop-indexed parameters $\{\boldsymbol{\alpha}_t^{(\tau)}, \beta_t^{(\tau)}, \mathbf{k}_t^{(\tau)}\}_{\tau=1}^{T}$, the cumulative state-transition operator becomes
\begin{equation}
\label{eq:Aeff}
\mathbf{A}_{t}^{\mathrm{eff}} \;=\; \prod_{\tau=1}^{T}\mathbf{A}_{t}^{(\tau)} 
\;=\; \prod_{\tau=1}^{T} \mathrm{Diag}\!\bigl(\boldsymbol{\alpha}_{t}^{(\tau)}\bigr)\Bigl(\mathbf{I} - \beta_{t}^{(\tau)}\mathbf{k}_{t}^{(\tau)}\mathbf{k}_{t}^{(\tau)\top}\Bigr) \;\in\; \mathbb{R}^{d_k \times d_k}.
\end{equation}

\paragraph{From rank-$1$ updates to rank-$T$ updates.}
A single $\mathbf{A}_t$ touches $\mathbf{S}_{t-1}$ along exactly one direction $\mathbf{k}_t$ (plus a channel decay), so it is a rank-$1$ correction to a diagonal map. Stacking $T$ such factors performs $T$ rank-$1$ corrections in succession; whether these corrections collapse or accumulate depends entirely on the geometry of the keys $\{\mathbf{k}_t^{(\tau)}\}_{\tau=1}^{T}$.

 The two extremes make the picture clear (these special cases of generalized Householder products are well known; see e.g.\ \citep{grazzi2025unlockingstatetrackinglinearrnns}):
\begin{itemize}
    \item \emph{Identical keys ($\mathbf{k}_t^{(1)} = \cdots = \mathbf{k}_t^{(T)} = \mathbf{k}$).} The product collapses: $\prod_{\tau=1}^{T} (\mathbf{I} - \beta_t^{(\tau)} \mathbf{k}\mathbf{k}^{\top}) = \mathbf{I} - \beta^{*} \mathbf{k}\mathbf{k}^{\top}$ for some scalar $\beta^*$. The rank stays at $1$ and looping buys no expressivity.
    \item \emph{Mutually orthogonal keys ($\mathbf{k}_t^{(i)\top} \mathbf{k}_t^{(j)} = 0$ for $i \neq j$).} The factors commute and combine cleanly into $\mathbf{I} - \sum_{\tau=1}^{T} \beta_t^{(\tau)} \mathbf{k}_t^{(\tau)} \mathbf{k}_t^{(\tau)\top}$, which is symmetric with rank exactly $T$ in the perturbation term.
\end{itemize}
The take-away is that loop-induced rank is governed by the geometry of the keys across loop steps. This bodes well in practice: in high-dimensional spaces ($d_k$ large), independently drawn unit vectors are nearly orthogonal with overwhelming probability. So for language modeling, where keys at different loop steps are generated from learned projections of the input, we should expect the keys to be approximately linearly independent and the effective rank of $\mathbf{A}_t^{\mathrm{eff}}$ to be close to $T$ rather than $1$.

\paragraph{Reflections, rotations, and arbitrary orthogonal maps.}
Stepping up from rank to geometry, we now compare the transformations reachable by a single block versus $T$ stacked blocks. With $\boldsymbol{\alpha}_t^{(\tau)} = \mathbf{1}$ and $\beta_t^{(\tau)} = 2$, each factor $\mathbf{I} - 2\mathbf{k}_t^{(\tau)} \mathbf{k}_t^{(\tau)\top}$ is a Householder reflection. A single block can therefore realize any reflection, but cannot represent a rotation (rotations have determinant $+1$, reflections have determinant $-1$). Loop unrolling lifts this restriction:

\begin{lemma}[Coordinate transpositions via reflection]
\label{lem:transposition}
The permutation matrix $\mathbf{P}_{(i,j)} \in \{0,1\}^{d_k \times d_k}$ that swaps coordinates $i$ and $j$ is realized by a single DPLR factor with $\boldsymbol{\alpha} = \mathbf{1}$, $\beta = 2$, and $\mathbf{k} = \tfrac{1}{\sqrt{2}}(\mathbf{e}_i - \mathbf{e}_j)$, where $\mathbf{e}_i, \mathbf{e}_j \in \mathbb{R}^{d_k}$ are standard basis vectors.
\end{lemma}

\begin{proof}
Direct expansion: $\mathbf{I} - 2\mathbf{k}\mathbf{k}^{\top} = \mathbf{I} - (\mathbf{e}_i - \mathbf{e}_j)(\mathbf{e}_i - \mathbf{e}_j)^{\top}$. Subtracting $(\mathbf{e}_i - \mathbf{e}_j)(\mathbf{e}_i - \mathbf{e}_j)^{\top}$ zeroes out the diagonal entries at $(i,i)$ and $(j,j)$ and places $1$'s at the off-diagonal entries $(i,j)$ and $(j,i)$, exactly producing $\mathbf{P}_{(i,j)}$.
\end{proof}

\begin{theorem}[Universal orthogonal representation]
\label{thm:orth}
Let $T \geq d_k$. For every orthogonal matrix $\mathbf{Q} \in \mathrm{O}(d_k)$, there exists a configuration of per-loop parameters such that $\mathbf{A}_t^{\mathrm{eff}} = \mathbf{Q}$.
\end{theorem}

\begin{proof}
By the Cartan--Dieudonn\'e theorem, every orthogonal matrix $\mathbf{Q} \in \mathrm{O}(d_k)$ can be written as a product of at most $d_k$ Householder reflections:
\[
\mathbf{Q} \;=\; \prod_{\tau=1}^{m} \bigl(\mathbf{I} - 2\,\mathbf{k}^{(\tau)} \mathbf{k}^{(\tau)\top}\bigr), 
\qquad m \leq d_k,
\]
for some unit vectors $\mathbf{k}^{(\tau)} \in \mathbb{R}^{d_k}$. Set $\boldsymbol{\alpha}_t^{(\tau)} = \mathbf{1}$, $\beta_t^{(\tau)} = 2$, and $\mathbf{k}_t^{(\tau)} = \mathbf{k}^{(\tau)}$ for $\tau \leq m$, and set $\beta_t^{(\tau)} = 0$ (identity factor) for $\tau > m$. Substituting into~\eqref{eq:Aeff} gives $\mathbf{A}_t^{\mathrm{eff}} = \mathbf{Q}$.
\end{proof}

In particular, since rotations are products of an even number of reflections, two looped blocks suffice to realize any 2D rotation, four for any 3D rotation, and so on. This is the geometric content of the upgrade from rank-$1$ to rank-$T$: looping the same DPLR block converts a reflection-only operator into one that covers the full orthogonal group.

\subsection{Loop \texorpdfstring{$\times$}{×} Sparse Attention: Receptive-Field Analysis}
\label{app:loop-swa}

We analyse the receptive field of an LT${2}$-SA layer that uses a static causal
sparse attention pattern $\mathcal{M}=\{\mathcal{M}_i\}_{i=1}^{N}$, with
$\mathcal{M}_i\subseteq\{1,\dots,i\}$ the keys visible to query $i$, looped for
$T$ iterations. Following Power attention and related discussion over receptive field~\citep{chen2025powerattentionexponentiallyscalingreceptive,xiao2025sliding}, we
distinguish two notions:
\begin{itemize}
\item the \emph{combinatorial} receptive field
      $\mathcal{I}_i^{(T)}\subseteq\{1,\dots,i\}$, defined as the set of input
      positions that \emph{can} influence the loop-$T$ output at $i$ —
      equivalently, the set of nodes with a directed path to $(i,T)$ in the
      layer-unrolled DAG induced by $\mathcal{M}$;
\item the \emph{effective} receptive field, defined as the set of input
      positions whose influence on the output is non-negligible once one
      accounts for softmax averaging and residual connections
      \citep{luo2017understandingeffectivereceptivefield}.
\end{itemize}
\Cref{ssec:combinatorial} bounds $|\mathcal{I}_i^{(T)}|$ for an exhaustive list
of static sparse patterns. \Cref{ssec:effective} then shows that residual
connections can collapse linear or even exponential combinatorial reach to a
constant effective horizon.

\subsubsection{Combinatorial receptive field (no residual)}
\label{ssec:combinatorial}

We treat the looped layer as a DAG: vertices $(i,t)$ with
$t\in\{0,\dots,T\}$, edges $(j,t-1)\to(i,t)$ iff $j\in\mathcal{M}_i$. Then
$\mathcal{I}_i^{(T)}=\{j:(j,0)\rightsquigarrow(i,T)\}$.

\paragraph{Sliding-window attention (SWA).}
For $\mathcal{M}_i=\{\max(1,i-w+1),\dots,i\}$:
\begin{proposition}[Linear receptive-field growth under SWA looping]
\label{prop:loop-swa}
For causal SWA with window $w$, looped $T$ times,
\[
  \mathcal{I}_{i}^{(T)} \;=\; \bigl\{\max(1,\,i-T(w-1)),\,\dots,\,i\bigr\},
  \qquad
  \bigl|\mathcal{I}_{i}^{(T)}\bigr| \;=\; \min\!\bigl(i,\,T(w-1)+1\bigr)
  \;=\; \mathcal{O}(Tw).
\]
\end{proposition}
\begin{proof}
By induction on $T$. For $T{=}1$, the claim restates the definition of a causal
window of size $w$. Assume the claim for $T{-}1$. The loop-$T$ state at
position $i$ is a function of the loop-$(T{-}1)$ states at positions
$j\in\{\max(1,i-w+1),\dots,i\}$, each of which by the inductive hypothesis
depends on inputs in $\{\max(1,j-(T{-}1)(w-1)),\dots,j\}$. The union over $j$
of these intervals equals $\{\max(1,i-T(w-1)),\dots,i\}$, completing the
induction.
\end{proof}

\paragraph{PowerAttention (power-of-two slashes).}
PowerAttention ~\citep{chen2025powerattentionexponentiallyscalingreceptive} couple SWA with $K=\lceil\log_2 N\rceil$ slash
heads at strides $2^{0},2^{1},\dots,2^{K-1}$, so that each query directly
attends to its local window plus the offsets $\{i-2^{k}:0\le k<K\}$. A single
loop already reaches distance $2^{K-1}$ via the slashes, and the $T$-fold
composition realises every distance expressible as a sum of $T$ powers of two,
giving
\[
  \mathcal{I}_{i}^{(T)} \;\supseteq\;
  \bigl\{i-d\,:\,0\le d\le \min(i-1,\,2^{T+K-1})\bigr\},
  \qquad
  \bigl|\mathcal{I}_{i}^{(T)}\bigr| \;=\;
  \mathcal{O}\!\bigl(\min(N,\,2^{T})\bigr).
\]
This is the unique pattern in our list whose combinatorial receptive field
grows \emph{exponentially} in $T$ \citep[Thm.~3.2]{chen2025powerattentionexponentiallyscalingreceptive},
while preserving the per-loop FLOP cost at $\mathcal{O}(Nw+N\log N)$,
comparable to plain SWA.

\subsubsection{Effective receptive field with residual connections}
\label{ssec:effective}

The bounds in \Cref{ssec:combinatorial} are upper bounds on \emph{topological}
reach; they do not capture how much of an actual influence signal survives
the cumulative softmax averaging and the residual short-circuit
$h_t^{(\ell)}=h_t^{(\ell-1)}+\mathrm{Attn}(h^{(\ell-1)})_t$ that every
Transformer block — and hence every loop iteration of an LT${2}$-SA layer —
applies. We adopt a
uniform-attention prior: each in-pattern key receives weight
$1/|\mathcal{M}_i|$ on average. Let $P_T(d)$ denote the influence of the input
at distance $d=i-j$ on the loop-$T$ output at $i$.

\paragraph{No residual: Gaussian dilution and $\sqrt{T}$ growth.}
Without residuals, $P_T$ is the $T$-fold convolution of the single-loop kernel
$P_1$. For SWA, $P_1$ is uniform on $[0,w-1]$ with mean $\mu_1=(w-1)/2$ and
variance $\sigma_1^{2}=(w^{2}-1)/12$. By a CLT-type argument, repeated
convolution drives $P_T$ to a Gaussian:
\[
  P_T(d)\;\approx\;\mathcal{N}\!\bigl(d;\;T\mu_1,\;T\sigma_1^{2}\bigr),
  \qquad
  D_{\text{eff}}^{\text{no-res}}(T)\;\approx\;0.58\,w\sqrt{T}.
\]
Hence even though $|\mathcal{I}_{i}^{(T)}|=\Theta(Tw)$, the influence
concentrates within an $\mathcal{O}(w\sqrt{T})$ band: information is diluted,
and the effective receptive field grows only \emph{sublinearly}. 

\paragraph{With residual: a depth-independent exponential horizon.}
Modeling the residual as a convex mixture is a tractable proxy for the
LayerNorm-induced effective contribution:
\[
  h_t^{(\ell)} \;=\; \alpha\,h_t^{(\ell-1)} \;+\; (1-\alpha)\,\mathrm{Attn}(h^{(\ell-1)})_t,
  \qquad \alpha\in(0,1),
\]
with $\alpha\in[0.9,0.99]$ typical at trained equilibrium, the influence
kernel becomes a \emph{spike-and-slab}: a mass of
$\alpha+\tfrac{1-\alpha}{|\mathcal{M}_i|}$ at $d=0$ and the residual share
$1-\alpha$ spread over $\mathcal{M}_i$. To travel a distance $d$ exceeding the
per-loop hop, information must take at least $\lceil d/w\rceil$ attention
hops, each multiplying surviving mass by $(1-\alpha)$, giving the exponential
upper bound
\[
  P_T(d) \;\le\; C\,(1-\alpha)^{\lceil d/w\rceil},
\]
which is asymptotically tight for $d\ll Tw$. Setting
$P_T(D_{\text{eff}})=\epsilon$ and solving yields a horizon
\emph{independent} of $T$:
\begin{equation}
\label{eq:effective-horizon}
  D_{\text{eff}}^{\text{res}}
  \;\approx\; w\cdot\frac{\ln(1/\epsilon)}{\ln\!\bigl(1/(1-\alpha)\bigr)}.
\end{equation}
For $\alpha=0.95$ and $\epsilon=10^{-2}$, this is $\approx 1.5\,w$, regardless
of how many times the layer is looped. Information from beyond
$\sim$2--3 window-widths is exponentially attenuated by the cumulative
residual mass, and \emph{additional loop iterations do not extend the effective
horizon}.

\section{Pre-training Setup}
\label{app:lm-setup}

This appendix describes the experimental setup behind the language modeling
results in Table~\ref{tab:downstream}. We pre-train every model from scratch
on the FineWeb-Edu corpus~\cite{penedo2024finewebdatasetsdecantingweb} at two parameter scales,
0.6B and 1.3B, under a fixed 100B-token budget. The implementation lives in
\texttt{apps/LT2} of our codebase and is built on the lingua
framework~\citep{meta_lingua}. All linear attention variants and NSA and from fla repo~\citep{yang2024fla}.

\paragraph{Data and tokenization.}
All runs draw from the FineWeb-Edu 100BT shard, packed at sequence length
$4096$. We tokenize with the Llama tiktoken tokenizer (vocabulary size
$128{,}256$) and prepend a BOS and append an EOS token to every document.
The data loader runs asynchronously with a prefetch buffer of $1024$ shards
and produces two views per example for downstream consumption.

\paragraph{Token budget.}
Every model is trained for $255{,}000$ optimizer steps at sequence length
$4096$. With the per-scale batch sizes given below, this corresponds to
roughly $100$ billion training tokens, i.e.\ a single epoch over the
FineWeb-Edu 100BT shard. The same token budget is used for every variant in
Table~\ref{tab:downstream}, so all comparisons are made at matched data.

\paragraph{Model scales.}
We use two parameter scales:
\begin{itemize}
  \item \textbf{0.6B}: hidden dimension $1024$, $25$ physical layers,
        $16$ attention heads, FFN multiple-of $256$, RoPE base
        $\theta = 10{,}000$.
  \item \textbf{1.3B}: hidden dimension $2048$, $16$--$25$ physical
        layers (depending on the layer mix described below), $16$
        attention heads, same FFN and RoPE settings.
\end{itemize}
We refer to the larger scale as ``1.3B'' throughout because the exact
parameter count varies slightly across variants (e.g.\ DSA layers are
heavier than dense Transformer blocks), but width and head count are held
fixed.

\paragraph{Looped variants.}
Every looped variant uses $T = 4$ loops over the physical layer stack: the
same parameters are applied four times in sequence, so the effective depth
is $4 \times n_{\text{layers}}$ while the parameter count stays at the
single-pass value. We add learned residual scaling between iterations
(\texttt{use\_residual=true}) and do not use cross-block residuals.
Non-looped baselines run a standard Transformer stack of the same physical
depth.

\paragraph{Layer mixes.}
The ``hybrid'' variants in Table~\ref{tab:downstream} interleave a
sub-quadratic mixer with periodic full softmax-attention layers in a
$4{:}1$ pattern: four mixer layers followed by one full-attention layer,
repeated. Concretely, for the 1.3B hybrid runs the physical stack has $16$
mixer layers and $4$ (or $3$) full-attention layers; for 0.6B the stack
has $17$ mixer layers and $4$ full-attention layers. The mixer family
varies by row and includes Gated DeltaNet (GDN), DeltaNet, KDA, NSA, RetNet,
HGRN2, Mamba2, MLA, and DSA. Pure baselines replace the mixer slots with
the same operator and remove the full-attention layers. Sub-quadratic
mixers use the \texttt{flash-linear-attention} kernels; full-attention
layers use FlashAttention-3, except for sliding-window variants which use
the FMHA path with \texttt{window\_left = w-1}, \texttt{window\_right = 0}
and a default window size of $w = 256$ at 0.6B and $w = 2048$ at 1.3B.

\paragraph{Optimization.}
We use AdamW with $(\beta_1, \beta_2) = (0.9, 0.95)$, gradient clipping at
norm $1.0$, and bf16 mixed-precision training. Peak learning rate is
$3\!\times\!10^{-4}$ for 1.3B runs and $1.5\!\times\!10^{-4}$ for 0.6B
runs (lowered for stability with the looped sliding-window 0.6B model).
The schedule is linear warmup followed by cosine decay to a floor of
$10^{-6}\times$ the peak rate. Warmup is $5{,}000$ or $10{,}000$ steps
depending on the variant. Weight decay is $0.1$ for almost all runs
($0.033$ for the 1.3B $4{:}1$ window baseline, where the smaller value was
needed to keep training stable). All runs use seed $777$ for the trainer
and seed $42$ for model initialization.

\paragraph{Distributed setup.}
Training is distributed with FSDP (\texttt{full\_shard}) and
\texttt{torch.compile} enabled wherever the kernel allows it (we disable
compilation only for DSA, whose TileLang kernel is incompatible with
\texttt{torch.compile}). We do not use tensor parallelism. Per-GPU
micro-batches of $2$--$12$ sequences are combined with gradient
accumulation between $1$ and $6$ steps to reach a global token-per-step
target of roughly $4\!\times\!10^{5}$ tokens; combined with $255{,}000$
steps this gives the $\sim$100B-token budget.

\paragraph{FLOP accounting.}
For looped models we count FLOPs over the unrolled depth, i.e.\ we
multiply the per-layer FLOP count by $n_{\text{layers}} \times T$ rather
than by $n_{\text{layers}}$ alone. This means the FLOP-per-token figures
we report for looped variants reflect the actual compute spent during
training, not the unique parameter count.

\paragraph{Evaluation.}
We report validation loss/perplexity on a held-out FineWeb-Edu 10BT
validation shard and zero-shot accuracy on the LM-Evaluation-Harness
suite: HellaSwag, BoolQ, PIQA, SocialIQA, WinoGrande, OpenBookQA,
ARC-Easy, ARC-Challenge, RACE, CommonsenseQA, and COPA. Numbers in
Table~\ref{tab:downstream} use the final checkpoint of each run.

\end{document}